\documentclass{article}
\usepackage[preprint]{neurips_2026}

\usepackage[utf8]{inputenc}
\usepackage[T1]{fontenc}
\usepackage{microtype}
\usepackage{nicefrac}
\usepackage{enumitem}

\usepackage{amsmath, amssymb, amsthm, mathtools}
\usepackage{amsfonts}
\usepackage{bm}

\usepackage{algorithm}
\usepackage{algpseudocode}
\algrenewcommand\algorithmicrequire{\textbf{Input:}}
\algrenewcommand\algorithmicensure{\textbf{Output:}}

\usepackage{graphicx}
\usepackage{booktabs}
\usepackage{wrapfig}

\usepackage{hyperref}
\usepackage{url}
\usepackage[capitalize, noabbrev]{cleveref}

\usepackage{xcolor}

\newcommand*\diff{\mathop{}\!\mathrm{d}}

\newcommand{\R}{\mathbb{R}}
\newcommand{\E}{\mathbb{E}}
\newcommand{\bbP}{\mathbb{P}}

\newcommand{\vx}{\bm{x}}
\newcommand{\vy}{\bm{y}}

\newcommand{\mA}{\bm{A}}

\newcommand{\mD}{\bm{D}}

\newcommand{\KL}[2]{D_{\mathrm{KL}}\!\left(#1\;\|\;#2\right)}
\newcommand{\Normal}[2]{\mathcal{N}\!\left(#1,\,#2\right)}
\newcommand{\Poisson}[1]{\mathrm{Poisson}(#1)}

\newcommand{\levy}{\nu}           
\newcommand{\tlevy}{\tilde{\nu}}  

\newcommand{\genP}{\mathcal{L}^P} 
\newcommand{\genQ}{\mathcal{L}^Q} 

\newtheorem{thm}{Theorem}[section]

\newtheorem{prop}[thm]{Proposition}
\newtheorem{cor}[thm]{Corollary}
\newtheorem{dfn}[thm]{Definition}
\newtheorem{remark}[thm]{Remark}

\title{Variational Inference for L\'{e}vy Process-Driven SDEs\\via Neural Tilting}

\author{
\textbf{Yaman K{\i}ndap}\textsuperscript{1},
\textbf{Manfred Opper}\textsuperscript{2},
\textbf{Benjamin Dupuis}\textsuperscript{3},
\textbf{Umut \c{S}im\c{s}ekli}\textsuperscript{3},
\textbf{Tolga Birdal}\textsuperscript{1}
\\[0.75em]
\small
\textsuperscript{1}Imperial College London, UK
\quad
\textsuperscript{2}Technical University of Berlin, Germany
\\
\small
\textsuperscript{3}INRIA, CNRS, Département d’Informatique de l’Ecole Normale Supérieure / PSL, France
}

\renewcommand{\paragraph}[1]{{\vspace{0.3mm}\noindent \bf #1}.}

\crefname{figure}{Fig.}{Figs.}
\Crefname{figure}{Fig.}{Figs.}
\crefname{table}{Tab.}{Tabs.}
\Crefname{table}{Tab.}{Tabs.}
\crefname{section}{Sec.}{Secs.}
\Crefname{section}{Sec.}{Secs.}
\crefname{equation}{Eq.}{Eqs.}
\Crefname{equation}{Eq.}{Eqs.}
\crefname{theorem}{Thm.}{Thms.}
\Crefname{theorem}{Thm.}{Thms.}
\crefname{remark}{Rem.}{Rems.}
\Crefname{remark}{Rem.}{Rems.}
\crefname{dfn}{Dfn.}{Dfns.}
\Crefname{dfn}{Dfn.}{Dfns.}
\crefname{algorithm}{Alg.}{Algs.}
\Crefname{algorithm}{Alg.}{Algs.}
\crefname{corollary}{Cor.}{Cors.}
\Crefname{corollary}{Corollary}{Corollaries}

\begin{document}

\maketitle

\vspace{-5mm}
\begin{abstract}
Modelling extreme events and heavy-tailed phenomena is central to building reliable predictive systems in domains such as finance, climate science, and safety-critical AI. While L\'{e}vy processes provide a natural mathematical framework for capturing jumps and heavy tails, Bayesian inference for L\'{e}vy-driven stochastic differential equations (SDEs) remains intractable with existing methods: Monte Carlo approaches are rigorous but lack scalability, whereas neural variational inference methods are efficient but rely on Gaussian assumptions that fail to capture discontinuities. We address this tension by introducing a neural exponential tilting framework for variational inference in L\'{e}vy-driven SDEs. Our approach constructs a flexible variational family by exponentially reweighting the L\'{e}vy measure using neural networks. This parametrization preserves the jump structure of the underlying process while remaining computationally tractable. To enable efficient inference, we develop a quadratic neural parametrization that yields closed-form normalization of the tilted measure, a conditional Gaussian representation for stable processes that facilitates simulation, and symmetry-aware Monte Carlo estimators for scalable optimization. Empirically, we demonstrate that the method accurately captures jump dynamics and yields reliable posterior inference in regimes where Gaussian-based variational approaches fail, on both synthetic and real-world datasets.
\end{abstract}

\section{Introduction}
\label{sec:intro}
Real-world stochastic systems, from financial markets to climate dynamics and
safety-critical AI, generate observations exhibit sudden discontinuities, asymmetric shocks, and
extreme events whose probability decays as a power law $x^{-\alpha}$ rather than
exponentially. This heavy-tailed structure is a defining property of many consequential
phenomena, not an anomalous departure from Gaussianity. Models that approximate it with
Gaussian noise are therefore systematically miscalibrated at the extremes, where the
cost of error is highest. The natural mathematical language for this full class of
stochastic behaviour is the theory of \emph{infinitely divisible
distributions} and associated L\'{e}vy processes, which form the complete
family of processes with independent and stationary increments~\citep{gnedenko1968limit,
ken1999levy}. Providing tractable approximate Bayesian inference over latent trajectories in L\'{e}vy process-driven
stochastic differential equations (SDEs) would therefore yield a principled,
uncertainty-aware framework for modelling temporal dynamics across any domain where the data-generating process departs from Gaussianity.

Neural SDE and ODE approaches have transformed continuous-time sequence modelling by
providing end-to-end differentiable, scalable frameworks competitive on standard
benchmarks~\citep{chen2018neural, tzen2019neural, kidger2020neural,
opper2019variational, daems2025efficient}, with recent extensions such as Neural Jump
SDEs~\citep{jia2019neural} and Neural MJD~\citep{gao2025neuralmjd} explicitly
incorporating discontinuities and achieving strong results on time series with abrupt
changes. Deep time-series forecasters, including DeepAR, DLinear, and
N-HiTS~\citep{salinas2020deepar, zeng2023transformers, challu2023nhits}, achieve
state-of-the-art performance on standard benchmarks through flexible, data-driven
parametrizations without any differential equation structure. Yet all of these methods
share a fundamental limitation. They model noise, including jump magnitudes, with
light-tailed distributions and fit their parameters by
maximum likelihood. This commits each model to a parametric noise structure calibrated
to average behaviour, yielding predictions that are systematically overconfident at
extreme quantiles, the failure mode that heavy-tailed phenomena directly expose, as shown in \cref{fig:synthetic_posterior}.

\begin{figure*}[t]
\centering
\includegraphics[width=\textwidth]{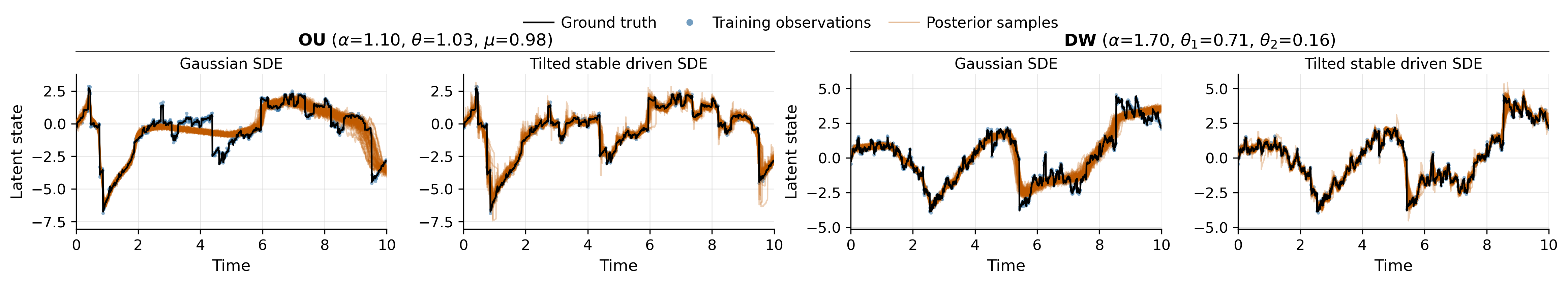}
\includegraphics[width=\textwidth]{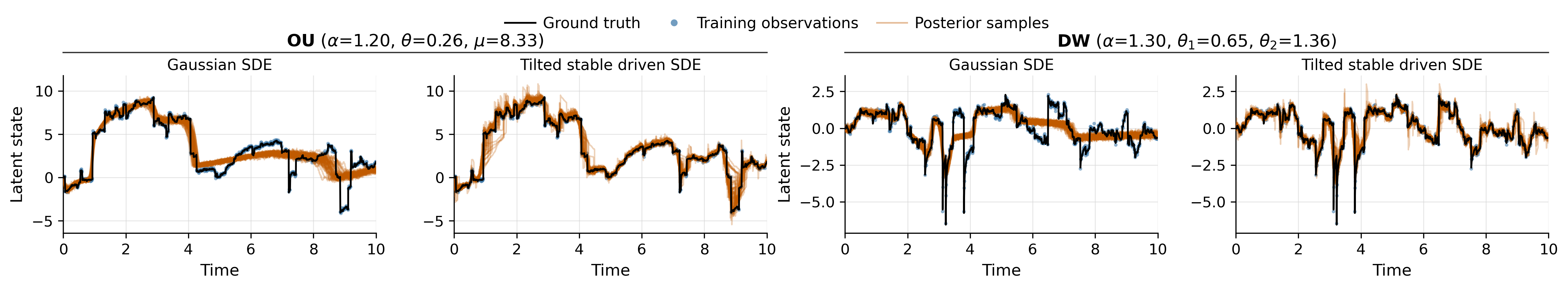}
\vspace{-5mm}
\caption{Posterior sample paths for two representative realisations of the OU (left pair) and double-well (right pair) systems, comparing the Gaussian SDE and our tilted-stable model.}
\label{fig:synthetic_posterior}
\vspace{-5mm}
\end{figure*}

This leaves a critical gap: no tractable framework exists for posterior inference over the latent dynamics of a L\'{e}vy-driven SDE, leaving jump structure as a fixed parametric assumption rather than a latent quantity to be inferred from data. Existing variational inference (VI) methods for stochastic processes obtain tractable path-space objectives through Brownian change-of-measure formulas, but these cannot produce posterior inference over the L\'{e}vy measure. Continuous variational processes are singular with respect to jump processes, and Brownian drift corrections in jump-diffusions leave the jump law unchanged. We close this gap by learning a neural exponential tilt of the prior L\'{e}vy measure, preserving the end-to-end differentiability and scalability of neural methods while yielding a tractable variational posterior with explicit jump and drift contributions to the ELBO.
\textbf{Our contributions are}:
\begin{itemize}[noitemsep,topsep=0em,leftmargin=*]
  \item A \textbf{path-space VI framework for L\'{e}vy-driven SDEs}
        that derives the optimal Markov variational posterior as an exponential tilt of
        the prior L\'{e}vy measure. 
  \item \textbf{Neural Tilting}: a quadratic neural parametrisation of the tilting
        function that preserves flexible state- and time-dependent posterior jump
        behaviour while yielding closed-form normalising constants.
  \item \textbf{Efficient simulation and ELBO optimisation algorithms} that exploit the
        conditionally Gaussian representation of stable processes and the symmetry of
        their L\'{e}vy measures.
  \item Empirical validation on synthetic and real-world datasets showing that the
        framework recovers heavy-tailed jump dynamics and improves tail calibration where
        Gaussian baselines fail.
\end{itemize}
Our source code will be made publicly available under: \href{https://circle-group.github.io/research/NeuralTilting/}{circle-group.github.io/research/NeuralTilting}.
\vspace{-3mm}
\section{Related work}
\label{sec:relwork}
\vspace{-3mm}

\paragraph{Simulation and data-driven methods for L\'{e}vy systems}
Numerical methods for L\'{e}vy-driven SDEs build on time-discretisation schemes~\cite{protter1997euler, fournier2011simulation}, with variance reduction via multilevel Monte Carlo~\cite{DEREICH20111565, jasra2019multilevel}, adaptive grids~\cite{kohatsu2010jump}, and exact simulation~\cite{pollock2016exact}. Continuous-time formulations~\cite{brockwell2001levy, kleinhans2007continuous} and conditionally Gaussian jump representations~\cite{Godsill2019,kindap2023generalised} improve efficiency; the latter technique is extended in our work from the prior L\'{e}vy measure to the exponentially tilted posterior. Data-driven approaches recover L\'{e}vy dynamics from observations via stable distribution properties~\cite{li2022extracting} and Koopman operator methods~\cite{lu2020discovering}. Despite exactness or arbitrarily small discretisation error, all of these methods lack end-to-end differentiability, precluding integration with modern ML pipelines.

\paragraph{Neural differential equation models}
Neural ODEs~\cite{chen2018neural} and their stochastic extensions provide scalable, end-to-end differentiable continuous-time sequence models, with scalable adjoint gradients~\cite{li2020scalable} and expressive latent architectures~\cite{kidger2020neural}. Extensions incorporate fractional white noise~\cite{tong2022fractional} and rough-path-driven jump dynamics~\cite{holberg2024exact}, and training efficiency has been substantially improved via stochastic optimal control~\cite{daems2025efficient}. All of these methods are restricted to continuous-path dynamics, precluding heavy-tailed jump behaviour.

\paragraph{Jump-diffusion and Markov jump process models}
A complementary line introduces jump structure into neural dynamical models: \cite{jia2019neural} extends neural ODEs with event-driven latent discontinuities, \cite{gao2025neuralmjd} combines a neural It\^{o} diffusion with a compound Poisson process, and \cite{zhang2024njdtpp} reformulates temporal point processes as neural jump-diffusion SDEs. All restrict jump magnitudes to light-tailed distributions, with no capacity to represent heavy-tailed L\'{e}vy geometry. A parallel line targets discrete-state Markov jump processes (MJPs) via variational mean-field approximations~\cite{opper2010approximate}, neural ODE rate functions~\cite{seifner2023neural}, expectation propagation~\cite{eich2025entropic}, and zero-shot foundation models~\cite{berghaus2024foundation}; these address discrete state spaces rather than continuous heavy-tailed dynamics.

\paragraph{Deep time-series forecasters}
Recurrent networks with parametric likelihoods~\cite{salinas2020deepar}, hierarchical interpolation~\cite{challu2023nhits}, linear models~\cite{zeng2023transformers}, and normalising flows conditioned on sequential context~\cite{rasul2021multivariate} achieve strong benchmark performance. Score-based diffusion models~\cite{rasul2021timegrad, tashiro2021csdi,nobis2024generative,nobis2025fractional} and interpretable diffusion architectures~\cite{yuan2024diffusionts} provide richer generative distributions. Distribution-free calibration via conformal prediction~\cite{stankeviciute2021conformal, xu2021conformal, zaffran2022adaptive} provides coverage guarantees independently of the noise model. All of these methods are restricted to light-tailed parametric noise distributions, with no capacity to represent heavy-tailed jump geometry.

\paragraph{Variational inference (VI) for stochastic processes}
VI for stochastic processes typically relies on the standard Girsanov theorem, which provides a tractable Radon--Nikodym derivative between path measures by changing the drift under a Brownian SDE, yielding a quadratic ELBO in closed form for Brownian motion~\cite{opper2019variational} and its fractional counterpart~\cite{daems2024variational}, with black-box variants using neural approximate posteriors~\cite{ryder2018blackbox, tzen2019neural}. 
For L\'{e}vy-driven SDEs, Fokker--Planck-based variational methods have been used to estimate drift functions under specified $\alpha$-stable noise~\cite{dai2021variational}, but their objective is drift recovery rather than posterior inference over latent paths or the L\'{e}vy measure. The resulting limitation is structural: existing stochastic process VI methods provide no variational mechanism for learning posterior heavy-tailed jump behaviour.

\vspace{-3mm}
\section{Background}
\label{sec:background}
\vspace{-3mm}
\begin{dfn}[L\'{e}vy process~\citep{applebaum2009levy}]
\label[dfn]{dfn:levy_process}
Let $(\Omega, \mathcal{F}, \bbP)$ be a probability space with filtration
$\{\mathcal{F}_t\}_{t \geq 0}$. An adapted stochastic process $\{L_t : t \geq 0\}$
with values in $\R^d$ is a \emph{L\'{e}vy process} if it satisfies: (i)~$L_0 = 0$ a.s.,
(ii)~independent increments, (iii)~stationary increments, and (iv)~stochastic continuity.
The stationarity and independence of increments suggest that L\'{e}vy processes may be
interpreted as a family of random walks in continuous time~\citep{bertoin1996levy,
ken1999levy, ContTankov2003}.
\end{dfn}

\begin{remark}[Connection to infinitely divisible distributions]
At any fixed time, the class of L\'{e}vy processes corresponds one-to-one with
infinitely divisible distributions (IDDs), which are probability measures $\mu$ on $\R^d$ that
can be written as the $n$-fold convolution $\mu = \mu_n * \cdots * \mu_n$ for any
positive integer $n$. This correspondence provides access to a rich family including
Gaussian (smooth diffusion), Poisson (discrete events), and stable processes
(self-similar, heavy-tailed). Brownian motion is the only L\'{e}vy process with
continuous sample paths; all other settings yield jump-diffusion
processes~\citep{ken1999levy}.
\end{remark}

\begin{dfn}[Characteristic function~\citep{Kallenberg_2002}]
\label[dfn]{dfn:char_fn}
A L\'{e}vy process $\{L_t\}_{t \geq 0}$ in $\R$ without a Brownian component is
defined through its characteristic function as
\begin{equation}
\label{eqn:levy-khintchine-repr}
    \E\!\left[e^{i u L_t}\right]
    = \exp\!\left(t \int_{\R_0} \bigl(e^{i u y} - 1 - i y \,\mathbf{1}_{|y|<1}\bigr)
      \levy(\diff y)\right),
\end{equation}
where $\levy$ is a L\'{e}vy measure on $\R_0 := \R\setminus\{0\}$ satisfying
$\int_{\R_0} \min(1, y^2)\,\levy(\diff y) < \infty$. This representation shows that the
distribution of a L\'{e}vy process is uniquely determined by its L\'{e}vy measure
$\levy$, so learning $\levy$ is sufficient to characterise the process.
\end{dfn}

\begin{remark}[Truncated L\'{e}vy measures]
\label[remark]{rem:truncation}
We focus on infinite-activity processes with $\int_{\R} \levy(\diff y) = \infty$.  In
practice, we work with a truncated measure restricted to $\R_\tau :=
\R\setminus(-\tau,\tau)$ for a threshold $\tau > 0$.  This removes the infinite
activity from small jumps while preserving the large jumps that characterise rare
discontinuities.  The contribution from small jumps below $\tau$ can be well-approximated
by Gaussian noise, providing a justification for this choice in practical
settings~\citep{AsmussenRosinski2001, godsill2024generalised}.
\end{remark}

\begin{dfn}[Symmetric $\alpha$-stable L\'{e}vy process]
\label[dfn]{dfn:stable}
The most prominent example of a heavy-tailed L\'{e}vy process is the symmetric
$\alpha$-stable process $\{L^\alpha_t\}$ with L\'{e}vy measure
\begin{equation}
\label{eq:levy_measure}
    \levy(\diff y) = |y|^{-1-\alpha}\,\diff y, \qquad \alpha \in (0, 2).
\end{equation}
The parameter $\alpha$ controls tail heaviness; stable processes have infinite variance
for all $\alpha$ and a finite mean only for $\alpha \in (1, 2)$.  The truncated version
is $\levy_\tau(\diff y) = |y|^{-1-\alpha}\,\mathbf{1}_{|y|\geq\tau}\,\diff y$, which
is finite since $\int_\tau^\infty y^{-1-\alpha}\,\diff y = \alpha^{-1}\tau^{-\alpha} <
\infty$.  Although $\levy_\tau$ is a finite measure, the $p$-th moment of the jump
size distribution remains infinite for $p \geq \alpha$, reflecting preservation of the
power-law tail under truncation.
\end{dfn}
\vspace{-2mm}

\vspace{-2mm}
\section{Variational Inference for L\'{e}vy processes (via Neural Tilting)}
\label{sec:method}
\vspace{-2mm}

We consider stochastic processes $X_{0:T} := \{X_t : 0 \leq t \leq T\}$ with
$X_t \in \R^d$ and fixed initial state $X_0 = x_0$, on Skorokhod spaces (the function space of right-continuous paths with
left limits, appropriate for processes with jump discontinuities) driven by general L\'{e}vy processes:
\begin{equation}
\label{eq:prior_sde}
    \diff X_t = f^\theta_t(X_t)\,\diff t + \diff L_t + \sigma(X_t)\,\diff B_t,
\end{equation}
where $f^\theta_t(\cdot)$ is a neural drift function with parameters $\theta$,
$\sigma(\cdot) \in \R^{d \times d}$ is a diffusion matrix with
$\mD(x) = \sigma(x)\sigma(x)^\top$, $\{L_t\}$ is a pure-jump L\'{e}vy process,
and $\{B_t\}$ is a standard Brownian motion.  For $d > 1$ we assume the driving
L\'{e}vy noise has \emph{independent components}, so the L\'{e}vy measure on $\R^d$
factorises as $\levy(\diff\vy) = \bigotimes_{i=1}^d \levy^{(i)}(\diff y_i)$; the fully
correlated case is a direction for future work.  Given noisy observations
$Y_{t_i} \sim \Normal{X_{t_i}}{\sigma_\varepsilon^2}$ at times
$\{t_i\}_{i=1}^n$, \textbf{our goal} is to approximate the posterior distribution
$p(X_{0:T} \mid Y_{0:T})$ and learn the parameters $\theta$.  We focus on the purely
non-Gaussian case $\sigma(\cdot) = 0$ in our experiments in order to isolate the
contribution of the jump structure, though the framework extends to joint jump-diffusion
dynamics.

Since the log-likelihood $\log p(Y_{0:T} \mid \theta)$ associated with noisy
observations generated from~\cref{eq:prior_sde} is intractable, we propose a
variational approach based on the evidence lower bound (ELBO):
\begin{equation}
\label{eq:elbo}
    \arg\sup_{Q}
    \left\{
      \E_{Q}\!\left[\sum_{i=1}^n \log p_\theta(Y_{t_i} \mid X_{t_i})\right]
      - \KL{Q}{P^\theta}
    \right\},
\end{equation}
where $P^\theta(X_{0:T})$ is the prior path measure induced by~\cref{eq:prior_sde} and the supremum is over trial
path measures $Q$, equaling the log-likelihood
$\log p_\theta(Y_{0:T} \mid \theta)$, attained when $Q = p_\theta(X_{0:T} \mid Y_{0:T})$.

\vspace{-2mm}
\subsection{Variational optimisation over Markov path measures}
\label{sec:variational_framework}
\vspace{-2mm}

We derive the variational family by optimising over Markov trial path measures, following the Doob h-transform structure of conditioned Markov processes~\cite{rogers2000diffusions}.

\begin{thm}[Optimal variational posterior family]
\label{thm:optimal_posterior_family}
Let $P^\theta$ be the prior path measure of the Lévy--diffusion SDE in~\cref{eq:prior_sde}, with generator $\mathcal L_t^P$ in~\cref{eq:prior_generator}. Optimising the ELBO in~\cref{eq:elbo} over Markov trial path measures yields the variational posterior SDE as
\begin{equation}
\label{eq:variational_sde}
    \diff X_t = \bigl(f^\theta_t(X_t) + \mD\nabla\phi_t(X_t)\bigr)\diff t
              + \sigma(X_t)\,\diff B_t
              + \diff L_t^\phi,
\end{equation}
where $\phi_t$ is the variational potential function, $\{L_t^\phi\}$ is the tilted L\'{e}vy process\footnote{$\{L_t^\phi\}$ is not a L\'{e}vy process since its jump measure depends on both time and the current state $X_t$.} with jump measure $\tlevy$ such that
\begin{equation}
\label{eq:tilted_measure}
    \tlevy(\diff\vy,t,X_t) = e^{\phi_t(X_t+\vy)-\phi_t(X_t)}\levy_\tau(\diff\vy).
\end{equation}
and $\mD = \sigma\sigma^\top$ is the diffusion matrix. For smooth test functions $G$, the posterior generator is 
\begin{equation}
\label{eq:tilted_generator}
\begin{aligned}
    \genQ_t G(x)
    &= \bigl(f_t^\theta(x) + \mD\nabla\phi_t(x)\bigr)^\top\!\nabla G(x)
     + \tfrac{1}{2}\mathrm{tr}\!\bigl(\mD\,\nabla^2 G(x)\bigr) \\
    &\quad + \int_{\R^d}\bigl(G(x+\vy)-G(x)\bigr)\,
             e^{\phi_t(x+\vy)-\phi_t(x)}\,\levy_\tau(\diff\vy).
\end{aligned}
\end{equation}
Compared to~\cref{eq:prior_generator}, the optimal variational family retains the same diffusion coefficient but acquires a Brownian drift correction $\mD\nabla\phi_t$ and a tilted L\'{e}vy measure. Setting $\phi_t \equiv 0$ in~\cref{eq:tilted_generator} recovers the prior generator $\genP_t$ of~\cref{eq:prior_generator} (see~\cref{app:variational_derivation}). For
pure-jump priors ($\sigma=0$), both the Brownian term and the drift correction $\mD\nabla\phi_t$ vanish, recovering the form used in \cref{sec:experiments}.
The result holds for any L\'{e}vy measure $\levy_\tau$.
\vspace{-3mm}
\end{thm}
\begin{proof}
    The detailed, step-by-step proof is provided in~\cref{app:variational_derivation}.
\end{proof}
\vspace{-2mm}
\begin{cor}[KL divergence between $Q$ and $P^\theta$]
For the variational posterior process in~\cref{thm:optimal_posterior_family}, the path-space KL divergence is 
\begin{equation}
\label{eq:kl_divergence}
    \KL{Q}{P^\theta}
    = \E_Q\!\left[
        \int_0^T\!\int_{\R^d} f(\vy,t,X_t)\,\levy_\tau(\diff\vy)\,\diff t
        + \frac{1}{2}\int_0^T \nabla\phi_t(X_t)^\top \mD\,\nabla\phi_t(X_t)\,\diff t
      \right],
\end{equation}
where $f(\vy,t,X_t) = H_t(X_t,\vy)\ln H_t(X_t,\vy) - H_t(X_t,\vy) + 1$,
$H_t(x,\vy) = e^{\phi_t(x+\vy)-\phi_t(x)}$. Under the factorised L\'{e}vy measure of~\cref{eq:prior_sde}, the jump integral
decomposes into $d$ independent scalar integrals.
For pure-jump priors ($\sigma = 0$, hence $\mD = 0$), the Brownian term vanishes.
\end{cor}
The framework extends naturally to priors with state-dependent jump sizes
$\gamma(X_{t^-})y$; see~\cref{app:state_dep} for details and a discussion of
the simulation adaptation this requires.

\vspace{-2mm}
\subsection{Quadratic parametrization for tractability}
\label{sec:quadratic}
\vspace{-2mm}

Practical implementation requires addressing the computational challenges of (a)
sampling from the tilted process $L_t^\phi$ and (b) evaluating the intractable
expectations in~\cref{eq:kl_divergence}. We address both through a quadratic
parametrization that retains neural flexibility while enabling analytical simplifications.

For clarity, we specialise to $d = 1$; the multivariate extension under
the independence assumption follows from~\cref{rem:multivariate}.
Consider the quadratic tilting function:
\begin{equation}
\label{eq:quadratic_phi}
    \phi_t(x') = A_t {x'}^2 + B_t x',
\end{equation}
where $A_t$ and $B_t$ are time-dependent scalar coefficients parametrized by neural
networks. For a jump from state $x$ to $x + y$, the tilting factor becomes:
\begin{equation}
\label{eq:quadratic_tilting}
    H_t(x, y) = e^{\phi_t(x+y) - \phi_t(x)}
    = \exp\!\bigl(A_t(2xy + y^2) + B_t y\bigr).
\end{equation}
The quadratic structure concentrates the time dependence into the scalar coefficients
$A_t$ and $B_t$, evaluated once per time step, while the dependence on state $x$ and
jump size $y$ takes a fixed quadratic form.  This makes $H_t(x,y)$ a Gaussian-type
function of $y$ for any fixed $t$ and $x$, enabling analytical computation of the
normalising constants required for simulation.

\begin{remark}[Multivariate extension]
\label[remark]{rem:multivariate}
For a $d$-dimensional process, the quadratic parametrization extends under the
assumption of independent state dimensions.  The tilting function decomposes as
$\phi_t(\vx') = \sum_{i=1}^d \phi_t^{(i)}(x'_i)
= \sum_{i=1}^d \bigl[A_t^{(i)}(x'_i)^2 + B_t^{(i)} x'_i\bigr]$,
where each dimension has its own scalar coefficients $A_t^{(i)}$ and $B_t^{(i)}$.
This is equivalent to the matrix form
$\phi_t(\vx') = \vx^{\prime\top}\mA_t\vx' + \bm{b}_t^\top\vx'$
with diagonal $\mA_t = \mathrm{diag}(A_t^{(1)},\ldots,A_t^{(d)})$ and
$\bm{b}_t = (B_t^{(1)},\ldots,B_t^{(d)})^\top$.
The tilting factor then factorises across dimensions as
$H_t(\vx,\vy) = \prod_{i=1}^d \exp\!\bigl(A_t^{(i)}(2x_i y_i + y_i^2) + B_t^{(i)} y_i\bigr)$,
so the KL integral and the normalising constants of~\cref{sec:simulation} remain
tractable dimension-by-dimension.
\end{remark}
\vspace{-2mm}
\paragraph{Adaptive temporal encoding}
The coefficients $A_t$ and $B_t$ are outputs of two MLPs whose shared input is a
learned embedding of time.  For a query time $t$, the embedding is
\begin{equation}
\label{eq:temporal_encoding}
    e(t) = \sum_{i=1}^{N} w_i(t)\,\bm{v}_i,
    \qquad
    w_i(t) = \frac{e^{-s\,|\,t-\tau_i\,|}}{\sum_{j=1}^N e^{-s\,|\,t-\tau_j\,|}},
\end{equation}
where $\{\tau_i\}_{i=1}^N \subset [0,T]$ are learnable reference times,
$\{\bm{v}_i\} \subset \R^{d_e}$ are learnable embedding vectors, and $s > 0$ is
a learnable sharpness (stored in log-space to ensure positivity).  This is a
Nadaraya--Watson estimator with a Laplacian kernel: the weights $w_i(t)$ decay
exponentially with temporal distance from each reference point.  Crucially, both
the reference locations and the embeddings are optimised jointly with the rest of
the model, so the $\tau_i$ migrate during training toward times at which the ELBO
is sensitive to the tilting, in practice toward the locations of jumps, without
any explicit supervision of jump timing.  The coefficients are then
$A_t = -(a_{\min} + \mathrm{softplus}(f_A(e(t)))) < 0$ and $B_t = f_B(e(t))$,
where $f_A$ and $f_B$ are MLPs and $a_{\min} > 0$ is a small fixed positive
constant that enforces a strict lower bound on $|A_t|$; the reparametrisation
guarantees strict negativity for all $t$, as required for integrability of the
tilted measure~(\cref{rem:tractability}).

\paragraph{Adaptive tempering and finite moment properties}
For symmetric $\alpha$-stable priors, the constraint $A_t < 0$, enforced by the parametrisation, ensures the tilted L\'{e}vy
measure is integrable, a necessary condition for the variational SDE to be well-posed.
The resulting exponentially truncated power-law tail behaviour is analogous to tempered
stable processes~\citep{rosinski2007tempering}, though here the tempering adapts with
time and state rather than being fixed, ensuring all moments of the tilted process are
finite. This lighter-tailed character is a consequence of the quadratic parametrisation:
for instance, $\phi_t$ for which $\phi_t(X_t+y) - \phi_t(X_t)$ remains bounded in $y$
yields a posterior with the same tail index $\alpha$ as the prior.

\vspace{-1.5mm}
\subsection{Approximation of the KL divergence}
\label{sec:kl_approx}
\vspace{-1.5mm}
With $M$ Monte Carlo (MC) paths $\{X^{(m)}_{t_j}\}$ drawn from the variational posterior
(\cref{sec:simulation}), the ELBO~\cref{eq:elbo} is approximated as
\begin{equation}
\label{eq:elbo_mc}
    \mathcal{L}(\theta,\phi) \approx \frac{1}{M}\sum_{m=1}^M \Biggl[
      \sum_{i=1}^n \log p_\theta\bigl(Y_{t_i} \mid X_{t_i}^{(m)}\bigr)
      - \sum_{j=0}^{N-1} \Delta t_j\;\hat{\mathcal{I}}\bigl(t_j, X_{t_j}^{(m)}\bigr)
    \Biggr],
\end{equation}
where $\hat{\mathcal{I}}(t_j, X_{t_j})$ is a Monte Carlo estimate of
$\mathcal{I}(t_j, X_{t_j}) \coloneqq \int_\R f(y,t,X_t)\,\levy_\tau(\diff y)$
at time $t_j$.  The remainder of this
subsection develops $\hat{\mathcal{I}}$; path simulation is detailed
in~\cref{sec:simulation}.  We exploit the symmetry of the truncated stable measure to
reduce the integration domain to $\mathcal{I}(t, X_t) = \int_\tau^\infty \bigl(f(y, t, X_t) + f(-y, t,
X_t)\bigr)\,y^{-1-\alpha}\,\diff y$. We express this as an expectation over jump sizes $y \sim p(y) = \alpha\tau^\alpha
y^{-1-\alpha}$ (supported on $[\tau,\infty)$), with normalisation constant
$C = \alpha^{-1}\tau^{-\alpha}$.  Samples from $p(y)$ are drawn via inverse-CDF: $y = \tau(1 - u)^{-1/\alpha}$ where $u \sim \mathrm{Uniform}(0, 1)$. \Cref{alg:kl_loss_approximation} summarizes the resulting MC procedure. The total computational cost is $\mathcal{O}(MNK)$, where $M$ is the number of MC paths, $N$ the number of time steps, and $K$ the number of jump samples.

\vspace{-2mm}
\subsection{Forward simulation of tilted L\'{e}vy processes}
\label{sec:simulation}
\vspace{-2mm}
The variational SDE~\cref{eq:variational_sde} requires forward simulation of the
tilted process $L_t^\phi$ with state- and time-dependent jump measure $\tlevy$.  We
exploit the \emph{conditionally Gaussian} structure of stable processes.

\begin{thm}[Conditionally Gaussian representation of tilted stable jumps]
\label{thm:cond_gaussian}
The tilted jump measure $\tlevy(y, t, X_t)$ admits a disintegration
$\tlevy(y, t, X_t) = \int_0^\infty \tilde{\sigma}(r; y, t, X_t)\,\tilde{\pi}_\tau(r, t, X_t)\,\diff r$,
where $\tilde{\sigma}(r; \cdot, t, X_t)$ is a probability kernel and $\tilde{\pi}_\tau$
is a truncated non-negative mixing measure:
\begin{equation}
\label{eq:levy_disintegration_terms}
    \tilde{\sigma}(r; y, t, X_t)
    = \frac{e^{\phi_t(X_t+y)-\phi_t(X_t)}\,\Normal{y}{r^2\sigma_G^2}}{C(r, t, X_t)},
    \quad
    \tilde{\pi}_\tau(r, t, X_t)
    = \frac{C(r, t, X_t)\,r^{-1-\alpha}}{K(\alpha, t, X_t)}\,\mathbf{1}_{r \geq \tau},
\end{equation}
with normalising constants $C(r, t, X_t)$ and $K(\alpha, t, X_t)$, and
$\sigma_G > 0$ the Gaussian scale of the variance-mean mixture
representation of the stable process.
\vspace{-1mm}
\end{thm}
\begin{proof}
The result extends the disintegration of $\alpha$-stable jump measures established
by~\cite{rosinski2001series,lemke2015fully} to include the tilting factor $e^{\phi_t(X_t+y)-\phi_t(X_t)}$.
See \cref{app:cond_gaussian} for a detailed proof.
\end{proof}

\begin{remark}[Analytical tractability under quadratic parametrization]
\label[remark]{rem:tractability}
Under~\cref{eq:quadratic_phi}, the normalising constant $C(r, t, X_t)$ of the
conditional kernel is available in closed form.  Setting $K_1 = 2A_t X_t + B_t$ and
$K_2 = A_t - \frac{1}{2r^2\sigma_G^2}$ ($K_2 < 0$ is enforced by $A_t < 0$,
as guaranteed by the parametrisation of \cref{sec:quadratic}),
$C(r, t, X_t) = \tfrac{1}{\sqrt{-2K_2 r^2\sigma_G^2}}\exp\!\bigl(-\tfrac{K_1^2}{4K_2}\bigr)$.
Thus, the conditional kernel $\tilde{\sigma}(r;\cdot,t,X_t)$ reduces to $\Normal{\mu_y(r)}{\sigma_y^2(r)}$ with $\mu_y(r) = -K_1/(2K_2)$ and
$\sigma_y^2(r) = -1/(2K_2)$.
\end{remark}

\begin{figure}[t]
\centering
\begin{minipage}[t]{0.42\linewidth}
\begin{algorithm}[H]
\small
\caption{KL approximation for tilted symmetric L\'{e}vy measures}
\label{alg:kl_loss_approximation}
\begin{algorithmic}[1]
\Require Process paths $\{X_{t_j}^{(m)}\}$; neural outputs $A_{t_j}$, $B_{t_j}$
\For{each time step $t_j$ and each path $m$}
  \State Evaluate $A_{t_j}$, $B_{t_j}$ from the neural networks.
  \State Draw $K$ samples $\{y_k\}_{k=1}^K$ via $y_k=\tau(1-u_k)^{-1/\alpha}$, $u_k\!\sim\!\mathrm{Uniform}(0,1)$.
  \State Compute $f_k^+ = f(y_k, t_j, X_{t_j}^{(m)})$ and
         $f_k^- = f(-y_k, t_j, X_{t_j}^{(m)})$.
  \State Approximate:
         $\mathcal{I}(t_j, X_{t_j}^{(m)}) \approx \frac{C}{K}\sum_{k=1}^K(f_k^+ + f_k^-)$.
\EndFor
\end{algorithmic}
\end{algorithm}
\end{minipage}
\hfill%
\begin{minipage}[t]{0.55\linewidth}
\begin{algorithm}[H]
\small
\caption{Tilted L\'{e}vy SDE simulation}
\label{alg:tilted_levy_simulation}
\begin{algorithmic}[1]
\Require Initial state $X_{t_0}$; time grid $\{t_j\}_{j=0}^{N-1}$; neural network parameters
\For{each time step $t_j$}
  \State Evaluate neural outputs $A_{t_j}$, $B_{t_j}$.
  \State Draw $K$ samples $\{y_k\}$ via $y_k=\tau(1-u_k)^{-1/\alpha}$, $u_k\!\sim\!\mathrm{Uniform}(0,1)$.
  \State Compute $\Lambda_j$ via~\cref{eq:intensity_mc}; draw $N_j \sim \Poisson{\Lambda_j}$.
  \For{$i = 1, \ldots, N_j$}
    \State Propose $r^* \sim r^{-1-\alpha}$; accept as $r^{(i)}$ with probability $C(r^*,t,X_t)/M(t,X_t)$.
    \State Draw jump $y^{(i)} \sim \Normal{\mu_y(r^{(i)})}{\sigma_y^2(r^{(i)})}$.
  \EndFor
  \State Update: $X_{t_{j+1}} = X_{t_j} + f^\theta_{t_j}(X_{t_j})\,\Delta t + \sum_{i=1}^{N_j} y^{(i)}$.
\EndFor
\end{algorithmic}
\end{algorithm}
\end{minipage}
\vspace{-4mm}
\end{figure}

\begin{thm}[Exact rejection sampler for the tilted mixing measure]
\label{thm:rejection_sampler}
Under the quadratic parametrisation of~\cref{sec:quadratic}, the tilted mixing
measure $\tilde{\pi}_\tau(r,t,X_t) \propto C(r,t,X_t)\,r^{-1-\alpha}$
of~\cref{thm:cond_gaussian} admits an exact rejection sampler with proposal
$q(r) \propto r^{-1-\alpha}$, $r \geq \tau$.  The normalising constant
$C(r,t,X_t)$ of~\cref{rem:tractability} satisfies the state-dependent envelope
\begin{equation}
\label{eq:envelope}
    C(r,t,X_t) \leq M(t,X_t) = \exp\!\left({K_1^2}/{4|A_t|}\right),
\end{equation}
so each proposal $r^*$ is accepted with probability $C(r^*,t,X_t)/M(t,X_t)$.
The sampler is exact and requires no discretisation, gradient computation, or
step-size tuning.
\vspace{-3mm}
\end{thm}
\begin{proof}
See~\cref{app:rejection_sampler}.
\end{proof}
\vspace{-3mm}
\paragraph{Jump intensity estimation}
For time interval $(t_j, t_{j+1})$ of length $\Delta t$, the total jump intensity is
approximated as
\begin{equation}
\label{eq:intensity_mc}
    \Lambda_j \approx \frac{\Delta t\,C}{K}
    \sum_{k=1}^K\bigl(H_t(X_{t_j}, y_k) + H_t(X_{t_j}, -y_k)\bigr),
\end{equation}
reusing the samples $\{y_k\}$ from the KL computation.  The number of jumps in the
interval is then $N_j \sim \Poisson{\Lambda_j}$.  \Cref{alg:tilted_levy_simulation}
assembles the complete forward simulation procedure.

The total computational cost of \cref{alg:tilted_levy_simulation} is
$\mathcal{O}(MN(K + \bar{N}_j))$ per iteration, where $\bar{N}_j$ is the expected
number of jumps per time interval; the rejection sampler adds a constant expected
factor per jump equal to the reciprocal of the mean acceptance probability.  In
training, the $K$ intensity samples are shared with
Alg.~\ref{alg:kl_loss_approximation}, reducing the incremental simulation cost
to $\mathcal{O}(MN\bar{N}_j)$.

\begin{table*}[t]
\caption{Synthetic experiment results ($\sigma_\varepsilon = 0.10$), reporting means
over 50 realisations per $\alpha$.  Held-out CRPS, mean absolute parameter recovery
errors, and jump CRPS at thresholds ($p_{97.5}, p_{99}$) for
the OU and double-well (DW) systems; lower is better.  Best result per metric in
\textbf{bold}.}
\label{tab:synthetic_main}
\centering
\scriptsize
\setlength{\tabcolsep}{2pt}
\resizebox{\textwidth}{!}{%
\begin{tabular}{l cc cc cc cc cc cc cc cc cc cc}
\toprule
& \multicolumn{10}{c}{Ornstein--Uhlenbeck}
& \multicolumn{10}{c}{Double well} \\
\cmidrule(lr){2-11}\cmidrule(lr){12-21}
& \multicolumn{2}{c}{CRPS $\downarrow$}
& \multicolumn{2}{c}{$|\hat{\theta}{-}\theta^*|$ $\downarrow$}
& \multicolumn{2}{c}{$|\hat{\mu}{-}\mu^*|$ $\downarrow$}
& \multicolumn{2}{c}{$p_{97.5}$ $\downarrow$}
& \multicolumn{2}{c}{$p_{99}$ $\downarrow$}
& \multicolumn{2}{c}{CRPS $\downarrow$}
& \multicolumn{2}{c}{$|\hat{\theta}_1{-}\theta_1^*|$ $\downarrow$}
& \multicolumn{2}{c}{$|\hat{\theta}_2{-}\theta_2^*|$ $\downarrow$}
& \multicolumn{2}{c}{$p_{97.5}$ $\downarrow$}
& \multicolumn{2}{c}{$p_{99}$ $\downarrow$} \\
\cmidrule(lr){2-3}\cmidrule(lr){4-5}\cmidrule(lr){6-7}\cmidrule(lr){8-9}\cmidrule(lr){10-11}
\cmidrule(lr){12-13}\cmidrule(lr){14-15}\cmidrule(lr){16-17}\cmidrule(lr){18-19}\cmidrule(lr){20-21}
$\alpha$ & G & TS & G & TS & G & TS & G & TS & G & TS & G & TS & G & TS & G & TS & G & TS & G & TS \\
\midrule
$1.1$ & $2.88$ & $\mathbf{0.89}$ & $4.63$ & $\mathbf{2.27}$ & $\mathbf{2.97}$ & $5.43$ & $2.15$ & $\mathbf{2.10}$ & $2.46$ & $\mathbf{2.41}$ & $0.41$ & $\mathbf{0.18}$ & $1.46$ & $\mathbf{0.76}$ & $1.06$ & $\mathbf{0.20}$ & $0.85$ & $\mathbf{0.60}$ & $1.35$ & $\mathbf{1.05}$ \\
$1.2$ & $1.63$ & $\mathbf{0.63}$  & $4.32$ & $\mathbf{1.98}$ & $\mathbf{3.01}$ & $6.11$ & $1.12$ & $\mathbf{0.56}$ & $1.23$ & $\mathbf{0.67}$ & $0.34$ & $\mathbf{0.14}$ & $1.81$ & $\mathbf{0.79}$ & $1.41$ & $\mathbf{0.33}$ & $0.58$ & $\mathbf{0.37}$ & $0.85$ & $\mathbf{0.62}$ \\
$1.3$ & $0.48$ & $\mathbf{0.37}$  & $5.25$ & $\mathbf{2.43}$ & $\mathbf{2.57}$ & $4.54$ & $\mathbf{2.35}$ & $2.36$ & $\mathbf{3.01}$ & $3.20$ & $0.37$ & $\mathbf{0.16}$ & $1.95$ & $\mathbf{1.05}$ & $1.72$ & $\mathbf{0.33}$ & $0.55$ & $\mathbf{0.35}$ & $0.77$ & $\mathbf{0.54}$ \\
$1.4$ & $0.72$ & $\mathbf{0.61}$  & $6.01$ & $\mathbf{2.63}$ & $\mathbf{2.66}$ & $4.84$ & $0.82$ & $\mathbf{0.79}$ & $\mathbf{1.25}$ & $1.26$ & $0.36$ & $\mathbf{0.15}$ & $1.72$ & $\mathbf{1.24}$ & $1.62$ & $\mathbf{0.28}$ & $0.60$ & $\mathbf{0.38}$ & $0.85$ & $\mathbf{0.60}$ \\
$1.5$ & $0.62$ & $\mathbf{0.45}$  & $6.11$ & $\mathbf{1.96}$ & $\mathbf{2.62}$ & $3.51$ & $0.86$ & $\mathbf{0.71}$ & $1.14$ & $\mathbf{0.98}$ & $0.34$ & $\mathbf{0.14}$ & $\mathbf{1.15}$ & $1.39$ & $1.07$ & $\mathbf{0.30}$ & $0.47$ & $\mathbf{0.29}$ & $0.61$ & $\mathbf{0.41}$ \\
$1.6$ & $0.41$ & $\mathbf{0.23}$  & $4.43$ & $\mathbf{2.04}$ & $\mathbf{2.57}$ & $3.56$ & $0.44$ & $\mathbf{0.33}$ & $0.52$ & $\mathbf{0.38}$ & $0.36$ & $\mathbf{0.14}$ & $1.42$ & $\mathbf{1.16}$ & $1.21$ & $\mathbf{0.27}$ & $0.47$ & $\mathbf{0.27}$ & $0.59$ & $\mathbf{0.38}$ \\
$1.7$ & $0.42$ & $\mathbf{0.28}$  & $5.79$ & $\mathbf{2.38}$ & $\mathbf{2.29}$ & $4.41$ & $0.52$ & $\mathbf{0.43}$ & $0.58$ & $\mathbf{0.48}$ & $0.37$ & $\mathbf{0.15}$ & $\mathbf{1.08}$ & $1.26$ & $0.76$ & $\mathbf{0.25}$ & $0.47$ & $\mathbf{0.28}$ & $0.56$ & $\mathbf{0.38}$ \\
$1.8$ & $0.36$ & $\mathbf{0.19}$  & $5.23$ & $\mathbf{2.16}$ & $\mathbf{1.72}$ & $2.20$ & $0.91$ & $\mathbf{0.28}$ & $0.98$ & $\mathbf{0.34}$ & $0.37$ & $\mathbf{0.15}$ & $\mathbf{1.11}$ & $1.19$ & $1.03$ & $\mathbf{0.30}$ & $0.47$ & $\mathbf{0.26}$ & $0.55$ & $\mathbf{0.34}$ \\
$1.9$ & $0.39$ & $\mathbf{0.22}$  & $4.99$ & $\mathbf{1.62}$ & $\mathbf{2.04}$ & $3.36$ & $0.51$ & $\mathbf{0.31}$ & $0.54$ & $\mathbf{0.33}$ & $0.38$ & $\mathbf{0.16}$ & $2.01$ & $\mathbf{1.14}$ & $1.49$ & $\mathbf{0.39}$ & $0.46$ & $\mathbf{0.26}$ & $0.52$ & $\mathbf{0.32}$ \\
\midrule
All $\alpha$ & $0.88$ & $\mathbf{0.43}$ & $5.19$ & $\mathbf{2.16}$ & $\mathbf{2.50}$ & $4.21$ & $1.08$ & $\mathbf{0.87}$ & $1.30$ & $\mathbf{1.12}$ & $0.37$ & $\mathbf{0.15}$ & $1.52$ & $\mathbf{1.11}$ & $1.26$ & $\mathbf{0.29}$ & $0.55$ & $\mathbf{0.34}$ & $0.74$ & $\mathbf{0.52}$ \\
\bottomrule
\end{tabular}%
}
\vspace{-4mm}
\end{table*}

\vspace{-1.5mm}
\section{Experiments}
\label{sec:experiments}
\vspace{-1.5mm}
We validate our framework on two settings that both exhibit pronounced heavy-tailed dynamics: (i) synthetic data with ground-truth parameters and (ii)
challenging real-world forecasting task.

\paragraph{Datasets}
In the \emph{synthetic setting}, data are generated from stable-process-driven SDEs
with known parameters and trajectories, enabling direct assessment of posterior quality and
parameter recovery.  We consider two drift families: a linear
Ornstein--Uhlenbeck (OU) system (parameters $\theta, \mu$) and a nonlinear double-well
potential (parameters $\theta_1, \theta_2$).\footnote{OU drift: $f^\theta(x) = \theta(\mu - x)$
with $\theta > 0$. Double-well drift: $f^\theta(x) = \theta_1 x - \theta_2 x^3$ with
$\theta_1, \theta_2 > 0$.}  For each stability index
$\alpha \in \{1.1, 1.2, \ldots, 1.9\}$ we generate 50 independent realisations with
randomly drawn drift parameters.  In the \emph{financial setting}, we work with
hourly log-prices of ten technology stocks (NVDA, GOOGL, MSFT, AAPL, AMZN, META,
TSLA, AMD, NFLX, INTC) spanning June~2024 to April~2026.  Models are trained on
rolling 30-day windows and evaluated on the subsequent 2-day forecast horizon,
yielding between 302 and 314 non-overlapping evaluation periods depending on ticker
availability.  We report results on both univariate (per-ticker) and multivariate
($d = 10$) forecasting tasks.

\paragraph{Baselines}
For the synthetic experiments we compare against a Gaussian SDE with an identical
drift parametrisation, isolating the effect of the noise model.  For the financial
experiments we additionally compare against
DeepAR~\citep{salinas2020deepar}, N-HiTS~\citep{challu2023nhits},
DLinear~\citep{zeng2023transformers}, Neural Jump SDEs~\citep{jia2019neural},
Neural MJD~\citep{gao2025neuralmjd}, and a Gaussian SDE.

\paragraph{Metrics}
Our primary metric is the Continuous Ranked Probability Score (CRPS)~\cite{matheson1976scoring}, which jointly
rewards sharpness and calibration and applies to any model that generates sample paths.
For synthetic experiments we additionally report mean absolute parameter recovery error
$|\hat{\theta} - \theta^*|$.  For financial experiments we supplement CRPS with MSE
and MAE to enable comparison with deterministic baselines, and report a
\emph{jump CRPS} computed on the subset of price increments exceeding the $p$-th
percentile for $p \in \{90, 95, 97.5, 99\}$ to specifically assess tail performance.
Uncertainty calibration is assessed via reliability diagrams.
Full implementation details, including wall-clock training times (\cref{tab:runtime}), are provided in~\cref{app:implementation}.

\vspace{-3mm}
\subsection{Synthetic experiments}
\label{sec:exp_synthetic}
\vspace{-2mm}

\noindent\textbf{Linear drift (Ornstein--Uhlenbeck)}
\cref{tab:synthetic_main} reports held-out CRPS, parameter recovery errors at
$\sigma_\varepsilon = 0.10$ and jump CRPS; additional details and results at $\sigma_\varepsilon = 0.05$ are in~\cref{app:extra_synthetic}.  Our model achieves lower CRPS than the Gaussian
baseline at every $\alpha$ value, with the improvement most pronounced for small
$\alpha$, reaching $3.2\times$ at $\alpha = 1.1$, and narrowing as $\alpha \to 2$ as
expected.  For the mean-reversion rate $\theta$, our model substantially reduces
recovery error ($2.16$ vs.\ $5.19$ averaged over all $\alpha$), confirming that
correctly modelling heavy-tailed noise prevents the drift from absorbing spurious jump
contributions; the global mean $\mu$ is harder to recover due to partial trade-offs
with the tilting function when observations are sparse.

\noindent\textbf{Non-linear drift (double well).}
The double-well system shows consistent CRPS improvement across all $\alpha$ ($0.15$
vs.\ $0.37$, a $2.4\times$ ratio).  The most striking result is for the cubic parameter
$\theta_2$, whose recovery error falls from $1.26$ to $0.29$, a $4.3\times$
improvement, confirming that without a correct heavy-tailed noise model the drift
absorbs rare large-amplitude excursions as apparent curvature.  Recovery of the linear
parameter $\theta_1$ is also improved overall ($1.11$ vs.\ $1.52$), with the tail
advantage confirmed at every threshold in~\cref{tab:synthetic_jump}.  Results at
$\sigma_\varepsilon = 0.05$ in~\cref{app:extra_synthetic} show the same pattern.

\vspace{-2mm}
\subsection{Financial forecasting}
\label{sec:exp_financial}
\vspace{-2mm}
We evaluate on NVDA hourly log-prices across 314 rolling 30-day training windows with a
2-day forecast horizon.  Each window proceeds in two phases.  In the training phase, the
ELBO is maximised over $\theta$ and $\phi$, serving as a system identification procedure,
with $\alpha$ fixed by grid search over $\{1.1, \ldots, 1.9\}$.  Forecasts are then
generated by the learned prior SDE (with learned $\theta$ and $\phi=0$), initialised from posterior samples at the window boundary. The Gaussian SDE follows the same protocol, isolating the noise model as the sole difference.
\cref{tab:financial_main} reports mean CRPS alongside jump CRPS at four tail
thresholds.  Results for GOOGL follow the same qualitative pattern, confirming that the
findings generalise across stocks; full results are provided in~\cref{app:extra_financial}.

\begin{figure}[t]
\centering
\begin{minipage}[t]{0.28\linewidth}
\centering
\includegraphics[width=\linewidth]{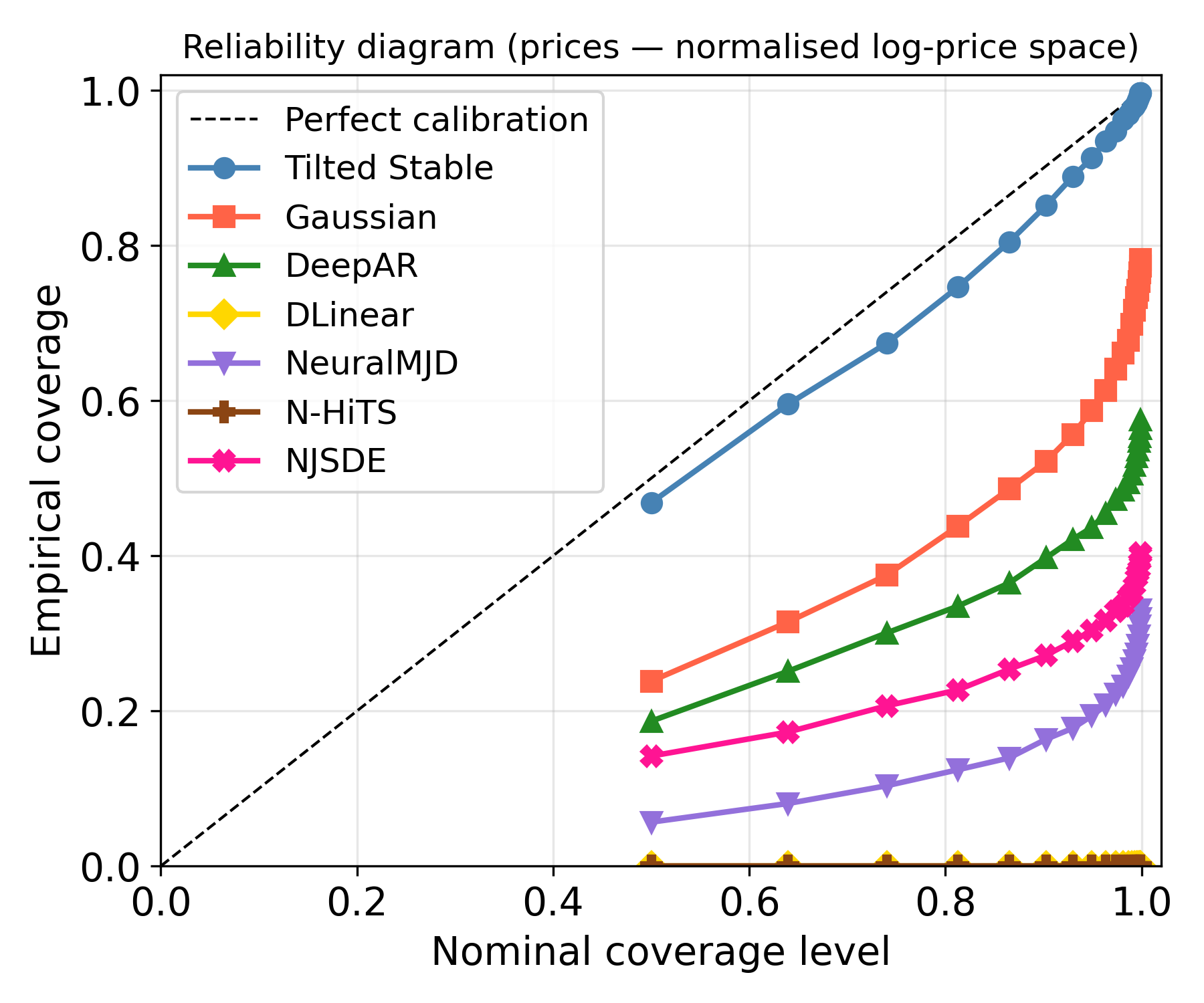}
\end{minipage}%
\hfill%
\begin{minipage}[t]{0.28\linewidth}
\centering
\includegraphics[width=\linewidth]{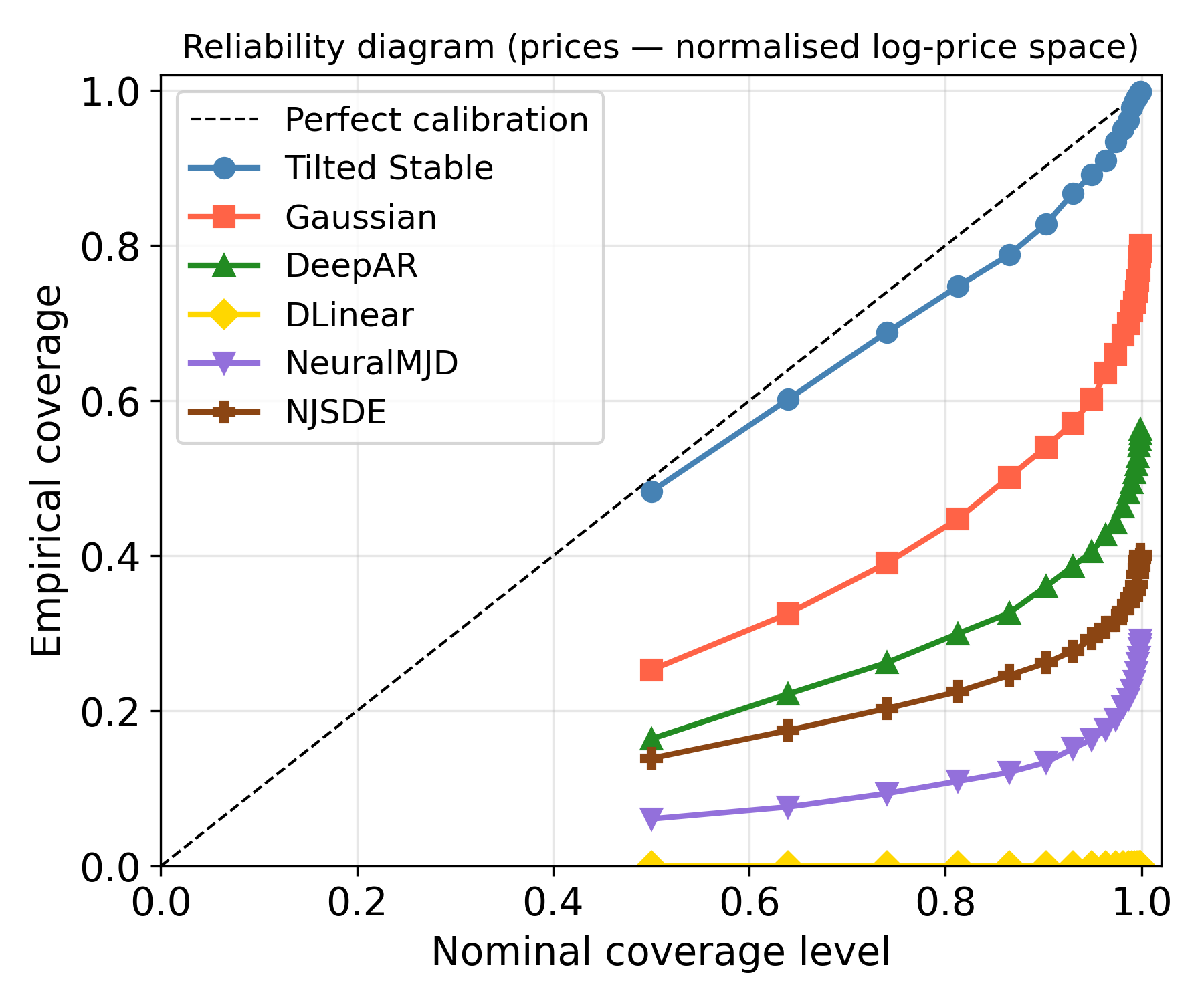}
\end{minipage}%
\hfill%
\begin{minipage}[t]{0.28\linewidth}
\centering
\includegraphics[width=\linewidth]{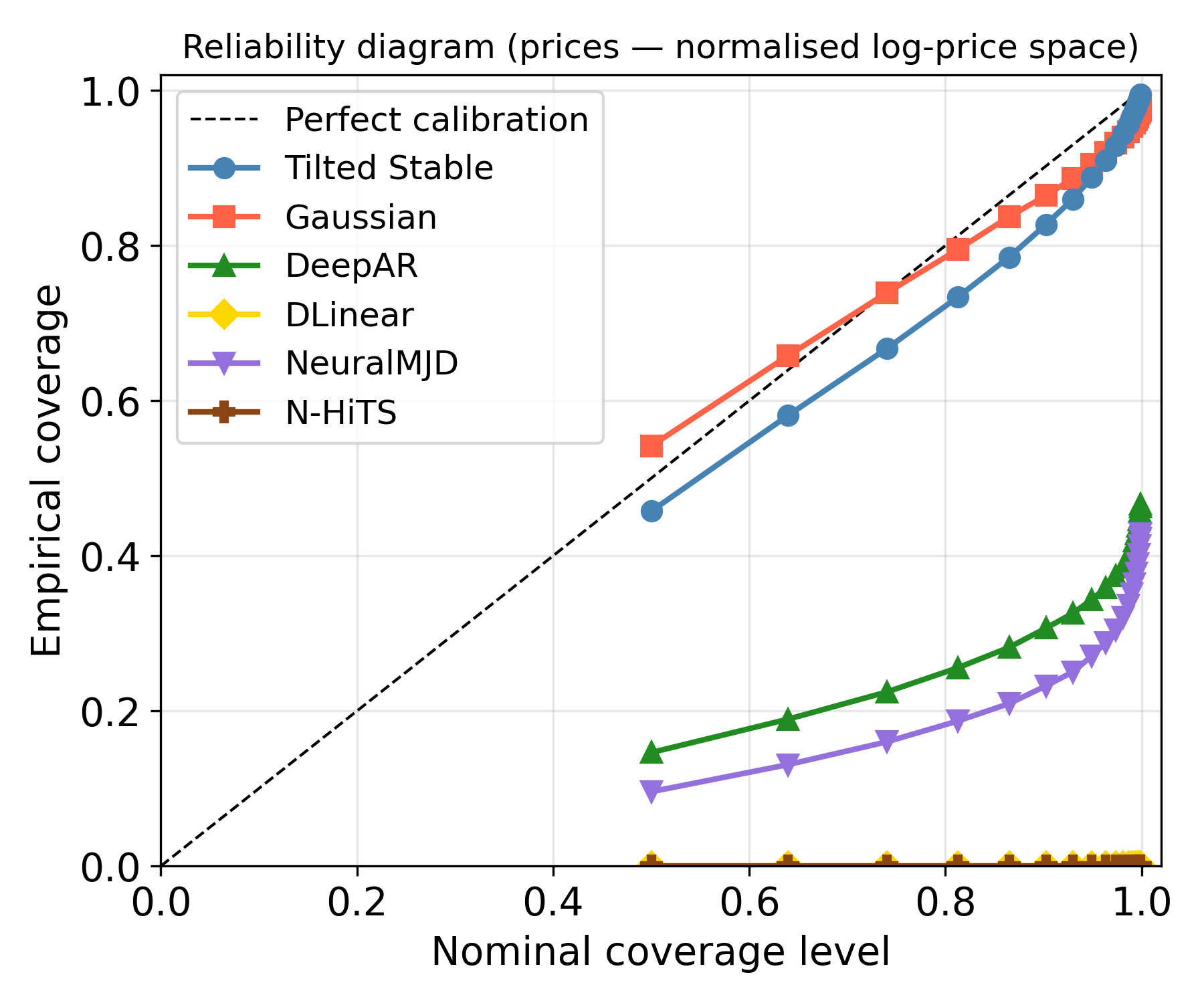}
\end{minipage}
\vspace{-2mm}
\caption{Reliability diagrams.  Empirical coverage vs.\ nominal level; the diagonal
denotes perfect calibration.  TS tracks the diagonal closely across all settings, while
all probabilistic baselines severely under-cover.
\textbf{Left:} NVDA univariate (314 evaluation windows).
\textbf{Centre:} GOOGL univariate (304 windows).
\textbf{Right:} Multivariate ($d{=}10$, 302 windows).\vspace{-4mm}}
\label{fig:reliability_financial}
\vspace{-2mm}
\end{figure}

\paragraph{Overall CRPS}
Our model achieves the lowest mean CRPS ($0.432$), a $21\%$ improvement over the
next-best probabilistic baseline (Neural MJD, $0.545$).  The Gaussian SDE, which shares
the same drift architecture and differs only in the noise model, scores $54\%$ worse
($0.663$), isolating the benefit of heavy-tailed noise.  The deterministic baselines
N-HiTS and DLinear produce point forecasts so their CRPS reduces to MAE; they rank fifth
and sixth.

\paragraph{Tail performance}
\cref{tab:financial_main} shows that our advantage is concentrated in the tails.
The gap over the next-best baseline (DeepAR) grows in absolute terms as the threshold
rises, reaching $0.277$ at $p_{99}$ compared to $0.186$ at $p_{90}$, consistent with
the synthetic finding that a correctly specified heavy-tailed noise model specifically
improves fit to rare large-amplitude moves.

\paragraph{Calibration}
Our model produces well-calibrated predictive intervals: empirical coverage closely
tracks the nominal level across the full range of prediction intervals, as illustrated
by the reliability diagram in~\cref{fig:reliability_financial}.  All probabilistic
baselines severely under-cover at every nominal level, indicating overconfident
predictive distributions.
\vspace{-1.5mm}

\begin{table*}[t]
\caption{Financial forecasting results; lower is better for all metrics.
Best in \textbf{bold}.  $^\dagger$Deterministic model; CRPS equals MAE.
Jump CRPS computed on price increments exceeding the indicated percentile threshold.
\textbf{Left:} Univariate (NVDA), 314 evaluation windows.
\textbf{Right:} Multivariate ($d{=}10$), 302 windows; NJ-SDE excluded (diverged in all runs).}
\label{tab:financial_main}%
\label{tab:multivariate_main}%
\noindent%
\begin{minipage}[t]{0.42\linewidth}
\centering
\resizebox{\linewidth}{!}{%
\setlength{\tabcolsep}{2pt}%
\begin{tabular}{l c cccc}
\toprule
& & \multicolumn{4}{c}{Jump CRPS $\downarrow$} \\
\cmidrule(lr){3-6}
Model & CRPS $\downarrow$ & $p_{90}$ & $p_{95}$ & $p_{97.5}$ & $p_{99}$ \\
\midrule
NJ-SDE             & $0.849{\pm}1.125$           & $1.385$          & $1.615$          & $1.769$          & $1.973$          \\
DLinear$^\dagger$  & $0.802{\pm}0.708$           & $1.357$          & $1.600$          & $1.637$          & $2.086$          \\
N-HiTS$^\dagger$   & $0.746{\pm}0.635$           & $1.406$          & $1.594$          & $1.690$          & $2.108$          \\
Gaussian SDE       & $0.663{\pm}0.984$           & $1.113$          & $1.233$          & $1.436$          & $1.929$          \\
DeepAR             & $0.556{\pm}0.509$           & $1.000$          & $1.145$          & $1.334$          & $1.749$          \\
Neural MJD         & $0.545{\pm}0.530$           & $1.059$          & $1.300$          & $1.495$          & $2.021$          \\
\textbf{TS (ours)} & $\mathbf{0.432}{\pm}0.427$ & $\mathbf{0.814}$ & $\mathbf{0.970}$ & $\mathbf{1.133}$ & $\mathbf{1.472}$ \\
\bottomrule
\end{tabular}%
}
\end{minipage}%
\hspace*{0.02\linewidth}%
\begin{minipage}[t]{0.56\linewidth}
\centering
\resizebox{\linewidth}{!}{%
\setlength{\tabcolsep}{2pt}%
\begin{tabular}{l cc cccc}
\toprule
& & & \multicolumn{4}{c}{Jump CRPS $\downarrow$} \\
\cmidrule(lr){4-7}
Model & CRPS $\downarrow$ & Energy $\downarrow$ & $p_{90}$ & $p_{95}$ & $p_{97.5}$ & $p_{99}$ \\
\midrule
N-HiTS$^\dagger$   & $0.690{\pm}0.297$           & $2.775$          & $1.031$          & $1.132$          & $1.290$          & $1.453$          \\
Gaussian SDE       & $0.556{\pm}0.241$           & $2.237$          & $0.766$          & $0.827$          & $0.926$          & $1.093$          \\
DeepAR             & $0.549{\pm}0.302$           & $2.222$          & $0.874$          & $0.952$          & $1.095$          & $1.261$          \\
DLinear$^\dagger$  & $0.526{\pm}0.260$           & $2.137$          & $0.850$          & $0.952$          & $1.135$          & $1.288$          \\
Neural MJD         & $0.504{\pm}0.260$           & $2.024$          & $0.799$          & $0.899$          & $1.061$          & $1.224$          \\
\textbf{TS (ours)} & $\mathbf{0.499}{\pm}0.247$ & $\mathbf{2.020}$ & $\mathbf{0.725}$ & $\mathbf{0.788}$ & $\mathbf{0.886}$ & $\mathbf{1.015}$ \\
\bottomrule
\end{tabular}%
}
\end{minipage}
\vspace{-3mm}
\end{table*}

\vspace{-3mm}
\subsection{Multivariate financial forecasting}
\label{sec:exp_multivariate}
\vspace{-3mm}
We evaluate on all ten stocks jointly ($d=10$) across 302 rolling evaluation windows,
using the energy score as the primary multivariate scoring rule alongside per-dimension
CRPS.  NJ-SDE diverged in every multivariate run and is excluded.~\cref{tab:multivariate_main} reports the results.

\paragraph{Results}
Our model achieves state-of-the-art average performance, matching Neural MJD on CRPS
($0.499$ vs $0.504$) and energy score ($2.020$ vs $2.024$), with the distinction
emerging in tail performance.

\paragraph{Tail performance}
\cref{tab:multivariate_main} shows our model leads at every jump CRPS threshold,
and the absolute gap over the next-best model (Gaussian SDE) widens from $0.041$ at
$p_{90}$ to $0.078$ at $p_{99}$, consistent with the pattern observed in both the
synthetic and univariate financial experiments.  Neural MJD, despite being near-tied on overall CRPS, drops to third on jump CRPS at
every threshold, suggesting its mixture approximation under-represents the most extreme
joint moves.

\paragraph{Calibration}
Our model remains well-calibrated in the multivariate setting ($50\%$ PI $\to$ $45.8\%$
empirical coverage; $90\%$ PI $\to$ $82.7\%$).  The Gaussian SDE is also reasonably
calibrated ($54.1\%$, $86.4\%$), while Neural MJD and DeepAR severely under-cover
($9.6\%$ and $14.6\%$ at the $50\%$ level respectively), indicating that their
probabilistic forecasts are overconfident.  Reliability diagrams for both GOOGL and the
multivariate setting are shown in~\cref{fig:reliability_financial}.
\vspace{-2mm}
\vspace{-3mm}
\section{Conclusion}
\label{sec:conclusion}
\vspace{-3.5mm}

We introduced a variational inference framework for L\'{e}vy-driven SDEs in which the
approximate posterior is expressed as a tilted L\'{e}vy process, with the prior jump
measure reweighted by a learned exponential factor and the Brownian component acquiring
the standard score correction. Although our experiments focus on stable processes, the
change-of-measure construction applies to the broader class of L\'{e}vy processes,
opening a route to scalable posterior inference in models with general non-Gaussian
jump structure. The quadratic neural parametrisation is the computational device that
makes this tractable, yielding closed-form normalising constants and an exact rejection
sampler for the tilted jump law.

Across all experimental settings the framework delivers $2$--$4\times$ CRPS
improvements over Gaussian baselines on synthetic systems, a $21\%$ improvement over
the next-best probabilistic model on financial data, and the strongest tail performance
at every jump-CRPS threshold in both univariate and multivariate forecasting.
Parameter recovery confirms that without a correctly specified heavy-tailed noise model,
the drift absorbs rare large-amplitude excursions as spurious curvature, distorting
inferred dynamics even when the overall fit appears reasonable.

\paragraph{Limitations and future work}
The quadratic parametrisation tempers posterior tails and may under-represent asymmetric
or multi-modal jump laws; richer flow-based tilting functions are a natural next step.
Our derivation uses a time-discretised change-of-measure argument; integrability conditions for a fully formal treatment are discussed in~\cref{rem:continuous_time_formal}. The multivariate formulation currently assumes independent stable dimensions; incorporating L\'{e}vy copulas would allow cross-dimensional tail dependence to be modelled explicitly. The stability index $\alpha$ is fixed by grid search; joint estimation within the ELBO remains future work.

\begin{ack}
T. B. was supported by a UKRI Future Leaders Fellowship (MR/Y018818/1). The authors acknowledge support from the UK AI Research Resource (AIRR Isambard AI) through grant 0251-4584-0945-1 - TopoFound. 
U.Ş. is partially supported by the French government under the management of Agence Nationale de la Recherche as part of the “Investissements d’avenir” program, reference ANR-19-P3IA-0001 (PRAIRIE 3IA Institute). B.D. and U.Ş. are partially supported by the European Research Council Starting Grant DYNASTY – 101039676.
\end{ack}

\bibliographystyle{plain}
\bibliography{bibliography}

\appendix

\clearpage
\appendix
\section*{Appendix}

\section{Variational derivation of the optimal Markov posterior}
\label{app:variational_derivation}

Both \cref{thm:optimal_posterior_family} and its Corollary derive from a single key
object: the \emph{optimal local transition ratio}, identified by constrained variational
optimisation over Markov path measures.  We first present a proof sketch covering both
results, then give complete step-by-step derivations.

\begin{proof}[Proof sketch]
We begin with a time-discretised decomposition of the KL divergence:
\begin{equation*}
    \KL{Q}{P^\theta} = \sum_{k=0}^{K-1}
    \int q_{t_k}(x)\,D_{\mathrm{KL}}^{t_{k+1},t_k}(x)\,\diff x,
\end{equation*}
where $D_{\mathrm{KL}}^{t_{k+1},t_k}(x) =
\int q_{t_{k+1},t_k}(x'|x)\ln\frac{q_{t_{k+1},t_k}(x'|x)}{p_{t_{k+1},t_k}(x'|x)}\diff x'$.
Incorporating the Chapman--Kolmogorov constraints via Lagrange multipliers $\phi_t(x)$,
the constrained variational optimisation yields the optimal transition ratio
\begin{equation}
\label{eq:optimal_transitions}
    \frac{q_{t+\Delta t,t}(x'|x)}{p_{t+\Delta t,t}(x'|x)}
    = \frac{e^{\phi_t(x')}}{\E_p\!\left[e^{\phi_t(X_{t+\Delta t})} \mid x\right]}.
\end{equation}

\textbf{\Cref{thm:optimal_posterior_family} (posterior generator and tilted L\'{e}vy
measure).}  
The infinitesimal generator of~\cref{eq:prior_sde} for sufficiently smooth test functions $G:\R^d\to\R$ is
\begin{equation}
\label{eq:prior_generator}
    \genP_t G(x)
    = f^\theta_t(x)^\top\nabla G(x)
    + \int_{\R^d}\bigl(G(x+\vy) - G(x)\bigr)\levy_\tau(\diff\vy)
    + \tfrac{1}{2}\mathrm{tr}\bigl(\mD(x)\,\nabla^2 G(x)\bigr).
\end{equation}
This generator characterisation
exists as L\'{e}vy processes are Markov, so their dynamics are completely specified
by their infinitesimal behaviour. When $\sigma=0$, the diffusion term vanishes and the
generator reduces to a drift plus a pure integral operator.

Using the optimal transition ratio to rewrite $\E_q[G(X_{t+\Delta t})\mid x]$
as a ratio of prior expectations, expanding numerator and denominator via the prior
generator expansion, and equating the $O(\Delta t)$ coefficient with $\genQ_t G(x)$ yields
the conjugate relation
\begin{equation*}
    \genQ_t G(x) = e^{-\phi_t(x)}\genP_t\!\bigl(e^{\phi_t(\cdot)}G(\cdot)\bigr)(x)
                 - G(x)\,e^{-\phi_t(x)}\genP_t e^{\phi_t(\cdot)}(x).
\end{equation*}
Substituting the prior generator~\cref{eq:prior_generator} into this relation and
expanding the jump and diffusion parts separately yields~\cref{eq:tilted_generator};
the tilted L\'{e}vy measure~\cref{eq:tilted_measure} is identified directly from the
jump integrand.

\textbf{Corollary (KL formula).}  Expanding the per-step KL using the same ratio and
taking $\Delta t\to 0$ gives
\begin{equation*}
    D_{\mathrm{KL}}^{t+\Delta t,t}(x)
    = \Delta t\Bigl\{
      e^{-\phi_t(x)}\genP_t\!\bigl(e^{\phi_t}\phi_t\bigr)(x)
      - \bigl(1+\phi_t(x)\bigr)e^{-\phi_t(x)}\genP_t e^{\phi_t(\cdot)}(x)
    \Bigr\} + o(\Delta t).
\end{equation*}
Substituting the generator~\cref{eq:prior_generator}, the jump integral contributes
$\int_{\R^d} f(\vy,t,x)\,\levy_\tau(\diff\vy)$ per unit time after simplification.  For the
diffusion part, let $\mathcal{L}^\mathrm{BM}$ denote the Brownian component of $\genP$.
A direct calculation using
$\mathcal{L}^\mathrm{BM}(e^\phi) =
e^\phi\bigl[\tfrac{1}{2}\mathrm{tr}(\mD\nabla^2\phi) + \tfrac{1}{2}\nabla\phi^\top\mD\nabla\phi\bigr]$
and $\mathcal{L}^\mathrm{BM}(e^\phi\phi) =
e^\phi\bigl[(1+\phi)\,\mathcal{L}^\mathrm{BM}\phi + \nabla\phi^\top\mD\nabla\phi\bigr]$
yields, after cancellation of the $(1+\phi)$ terms, the residual contribution
$\tfrac{1}{2}\nabla\phi_t(x)^\top\mD\,\nabla\phi_t(x)$ per unit time.
Summing over time steps, converting to path integrals under $Q$, and taking
$\Delta t\to 0$ yields~\cref{eq:kl_divergence}.
\end{proof}

We now give complete step-by-step derivations.  Step~1 is common to both results;
Step~2 proves \cref{thm:optimal_posterior_family}; Step~3 proves the Corollary.
Throughout, $\phi \equiv \phi_t(x)$ and all generator actions are evaluated at a
fixed $x$ unless otherwise noted.

\paragraph{Step 1: Optimal transition ratio}
We minimise the KL divergence subject to the constraint that $Q$ corresponds to a
valid Markov process.  The Chapman--Kolmogorov equations require
\begin{equation*}
    q_{t+\Delta t}(x') = \int q_{t+\Delta t,t}(x'|x)\,q_t(x)\,\diff x.
\end{equation*}
Introducing Lagrange multipliers $\phi_t(x')$ for each constraint and taking the
functional derivative of the augmented objective with respect to $q_{t+\Delta t,t}(x'|x)$
yields the first-order condition
\begin{equation*}
    \ln\frac{q_{t+\Delta t,t}(x'|x)}{p_{t+\Delta t,t}(x'|x)} = \phi_t(x') - \ln Z_t(x),
\end{equation*}
where $Z_t(x) = \E_p\!\left[e^{\phi_t(X_{t+\Delta t})}\mid x\right]$ is the
normalising constant imposed by the constraint.  Exponentiating gives the optimal
transition ratio~\cref{eq:optimal_transitions}.

\paragraph{Step 2: Proof of \cref{thm:optimal_posterior_family}}
We derive the posterior generator from the optimal transition ratio.  For any test
function $G$, the posterior generator is defined by
\begin{equation*}
    \E_q\!\left[G(X_{t+\Delta t})\mid x\right] = G(x) + \Delta t\,\genQ_t G(x) + o(\Delta t).
\end{equation*}
Using the optimal transition ratio~\cref{eq:optimal_transitions}, we can also write
this expectation as a ratio of prior expectations:
\begin{equation*}
    \E_q\!\left[G(X_{t+\Delta t})\mid x\right]
    = \frac{\E_p\!\left[e^{\phi_t(X_{t+\Delta t})}G(X_{t+\Delta t})\mid x\right]}
           {\E_p\!\left[e^{\phi_t(X_{t+\Delta t})}\mid x\right]}.
\end{equation*}
Applying the prior generator expansion $\E_p[G(X_{t+\Delta t})\mid x] =
G(x) + \Delta t\,\genP_t G(x) + o(\Delta t)$ to numerator and denominator:
\begin{equation*}
    \E_q\!\left[G(X_{t+\Delta t})\mid x\right]
    = \frac{e^{\phi}G(x) + \Delta t\,\genP_t\!\bigl(e^{\phi_t(\cdot)}G(\cdot)\bigr)(x)
            + o(\Delta t)}
           {e^{\phi} + \Delta t\,\genP_t e^{\phi_t(\cdot)}(x) + o(\Delta t)}.
\end{equation*}
Factoring out $e^\phi$ and performing a first-order Taylor expansion of the denominator,
$(1 + \Delta t\,e^{-\phi}\genP_t e^{\phi_t}(x))^{-1} =
1 - \Delta t\,e^{-\phi}\genP_t e^{\phi_t}(x) + o(\Delta t)$, then expanding and
retaining $O(\Delta t)$ terms gives
\begin{equation*}
    \E_q\!\left[G(X_{t+\Delta t})\mid x\right]
    = G(x) + \Delta t\Bigl[
        e^{-\phi}\genP_t\!\bigl(e^{\phi_t}G\bigr)(x)
        - G(x)\,e^{-\phi}\genP_t e^{\phi_t}(x)
      \Bigr] + o(\Delta t).
\end{equation*}
Equating the $O(\Delta t)$ coefficient with $\genQ_t G(x)$ yields the conjugate relation
\begin{equation}
\label{eq:conjugate_relation}
    \genQ_t G(x) = e^{-\phi}\genP_t\!\bigl(e^{\phi_t(\cdot)}G(\cdot)\bigr)(x)
                 - G(x)\,e^{-\phi}\genP_t e^{\phi_t(\cdot)}(x).
\end{equation}
We now substitute the prior generator $\genP_t = \genP_J + \genP_\mathrm{BM} +
f_t^\theta\cdot\nabla$ into~\cref{eq:conjugate_relation}.

\emph{Jump part.}  Setting $H \equiv H_t(x,y) = e^{\phi_t(x+y)-\phi}$:
\begin{align*}
    e^{-\phi}\genP_J(e^\phi G) - G\,e^{-\phi}\genP_J e^\phi
    &= \int_\R\bigl[HG(x+y) - G(x)\bigr]\levy_\tau(\diff y)
     - G\int_\R(H-1)\levy_\tau(\diff y) \\
    &= \int_\R H\bigl(G(x+y)-G(x)\bigr)\levy_\tau(\diff y) \\
    &= \int_\R\bigl(G(x+y)-G(x)\bigr)e^{\phi_t(x+y)-\phi_t(x)}\levy_\tau(\diff y).
\end{align*}
The tilted L\'{e}vy measure is identified directly from this integrand as
$\tlevy(\diff y,t,x) = e^{\phi_t(x+y)-\phi_t(x)}\levy_\tau(\diff y)$,
establishing~\cref{eq:tilted_measure}.

\emph{Diffusion part.}  Since $\partial_i\partial_j(e^\phi G) =
e^\phi[\partial_i\phi\partial_j\phi G + \partial_i\partial_j\phi G +
\partial_i\phi\partial_j G + \partial_j\phi\partial_i G + \partial_i\partial_j G]$:
\begin{equation*}
    e^{-\phi}\genP_\mathrm{BM}(e^\phi G) - G\,e^{-\phi}\genP_\mathrm{BM}(e^\phi)
    = \tfrac{1}{2}\sum_{ij}D_{ij}\bigl[
        \partial_i\phi\,\partial_j G + \partial_j\phi\,\partial_i G
        + \partial_i\partial_j G
      \bigr].
\end{equation*}
With $\mD$ symmetric, $\sum_{ij}D_{ij}\partial_j\phi\,\partial_i G =
(\mD\nabla\phi)^\top\nabla G$, so the diffusion contribution is
$(\mD\nabla\phi)^\top\nabla G + \tfrac{1}{2}\mathrm{tr}(\mD\,\nabla^2 G)$.

\emph{Drift part.}  The drift term passes through the conjugate relation unchanged, since
$e^{-\phi}(f_t^\theta\cdot\nabla)(e^\phi G) - G\,e^{-\phi}(f_t^\theta\cdot\nabla)e^\phi
= f_t^\theta\cdot\nabla G$.

Adding all three contributions confirms~\cref{eq:tilted_generator}.

\paragraph{Step 3: Proof of the Corollary}
With the optimal transition ratio in hand, the per-step KL is
\begin{equation*}
    D_{\mathrm{KL}}^{t+\Delta t,t}(x)
    = \E_q\!\left[\phi_t(X_{t+\Delta t})\mid x\right] - \ln Z_t(x).
\end{equation*}
Applying the generator expansion
$\E_p[G(X_{t+\Delta t})\mid x] = G(x) + \Delta t\,\genP_t G(x) + o(\Delta t)$
to both terms:
\begin{align*}
    \E_q\!\left[\phi_t(X_{t+\Delta t})\mid x\right]
    &= \frac{\E_p\!\left[e^{\phi_t}
       \phi_t(X_{t+\Delta t})\mid x\right]}{\E_p\!\left[e^{\phi_t(X_{t+\Delta t})}\mid x\right]} \\
    &= \phi + \Delta t\Bigl[
        e^{-\phi}\genP_t(e^{\phi_t(\cdot)}\phi_t(\cdot))(x)
        - \phi\,e^{-\phi}\genP_t e^{\phi_t(\cdot)}(x)
      \Bigr] + o(\Delta t), \\[4pt]
    \ln Z_t(x)
    &= \phi + \Delta t\,e^{-\phi}\genP_t e^{\phi_t(\cdot)}(x) + o(\Delta t).
\end{align*}
Subtracting gives
\begin{equation}
\label{eq:perstep_kl}
    D_{\mathrm{KL}}^{t+\Delta t,t}(x)
    = \Delta t\Bigl\{
        e^{-\phi}\genP_t(e^{\phi_t}\phi_t)(x)
        - (1+\phi)\,e^{-\phi}\genP_t e^{\phi_t}(x)
      \Bigr\} + o(\Delta t).
\end{equation}

\emph{Jump contribution.}
Let $\genP_J G(x) = \int_\R (G(x+y)-G(x))\,\levy_\tau(\diff y)$ be the jump part
of the generator.  We evaluate each term in~\cref{eq:perstep_kl} for $\genP_J$:
\begin{align*}
    e^{-\phi}\genP_J(e^{\phi_t}\phi_t)(x)
    &= \int_\R\Bigl[e^{\phi_t(x+y)-\phi}\phi_t(x+y) - \phi\Bigr]\levy_\tau(\diff y), \\
    (1+\phi)\,e^{-\phi}\genP_J e^{\phi_t}(x)
    &= (1+\phi)\int_\R\Bigl[e^{\phi_t(x+y)-\phi} - 1\Bigr]\levy_\tau(\diff y).
\end{align*}
Setting $H \equiv H_t(x,y) = e^{\phi_t(x+y)-\phi}$ and subtracting:
\begin{align*}
    e^{-\phi}\genP_J(e^{\phi_t}\phi_t)(x)
    &- (1+\phi)\,e^{-\phi}\genP_J e^{\phi_t}(x) \\
    &= \int_\R\Bigl[H\phi_t(x+y) - \phi - (1+\phi)(H-1)\Bigr]\levy_\tau(\diff y) \\
    &= \int_\R\Bigl[H\phi_t(x+y) - H\phi - H + 1\Bigr]\levy_\tau(\diff y) \\
    &= \int_\R\Bigl[H(\phi_t(x+y)-\phi) - H + 1\Bigr]\levy_\tau(\diff y) \\
    &= \int_\R\Bigl[H\ln H - H + 1\Bigr]\levy_\tau(\diff y)
    = \int_\R f(y,t,x)\,\levy_\tau(\diff y),
\end{align*}
where the penultimate equality uses $\ln H = \phi_t(x+y) - \phi$.

\emph{Diffusion contribution.}
Let $\genP_\mathrm{BM} G(x) = \tfrac{1}{2}\mathrm{tr}(\mD\,\nabla^2 G(x))$ be the
Brownian part.  We compute each factor by applying the product rule twice.

For $\genP_\mathrm{BM}(e^\phi)$: since $\partial_i\partial_j e^\phi =
e^\phi(\partial_i\phi\partial_j\phi + \partial_i\partial_j\phi)$,
\begin{equation*}
    e^{-\phi}\genP_\mathrm{BM}(e^\phi)
    = \tfrac{1}{2}\nabla\phi^\top\!\mD\,\nabla\phi
    + \tfrac{1}{2}\mathrm{tr}(\mD\,\nabla^2\phi).
\end{equation*}
For $\genP_\mathrm{BM}(e^\phi\phi)$: since
$\partial_j(e^\phi\phi) = e^\phi\partial_j\phi(1+\phi)$, applying $\partial_i$ gives
$\partial_i\partial_j(e^\phi\phi) = e^\phi[\partial_i\phi\partial_j\phi(2+\phi) +
\partial_i\partial_j\phi(1+\phi)]$, so
\begin{equation*}
    e^{-\phi}\genP_\mathrm{BM}(e^\phi\phi)
    = (2+\phi)\,\tfrac{1}{2}\nabla\phi^\top\!\mD\,\nabla\phi
    + (1+\phi)\,\tfrac{1}{2}\mathrm{tr}(\mD\,\nabla^2\phi).
\end{equation*}
Subtracting $(1+\phi)$ times the first from the second, the
$\mathrm{tr}(\mD\,\nabla^2\phi)$ terms cancel exactly:
\begin{equation*}
    e^{-\phi}\genP_\mathrm{BM}(e^\phi\phi) - (1+\phi)\,e^{-\phi}\genP_\mathrm{BM}(e^\phi)
    = \tfrac{1}{2}(2+\phi-(1+\phi))\,\nabla\phi^\top\!\mD\,\nabla\phi
    = \tfrac{1}{2}\nabla\phi^\top\!\mD\,\nabla\phi.
\end{equation*}

\emph{Combining and taking the limit.}
Adding the jump and diffusion contributions, the full per-step KL~\cref{eq:perstep_kl} becomes
\begin{equation*}
    D_{\mathrm{KL}}^{t+\Delta t,t}(x)
    = \Delta t\left\{
        \int_\R f(y,t,x)\,\levy_\tau(\diff y)
        + \tfrac{1}{2}\nabla\phi_t(x)^\top\!\mD\,\nabla\phi_t(x)
      \right\} + o(\Delta t).
\end{equation*}
Summing over time steps and converting the sum to a path integral under $Q$ via
$\sum_k q_{t_k}(x)\Delta t \to \E_Q[\,\cdot\,]$ in the limit $\Delta t\to 0$
yields~\cref{eq:kl_divergence}.

\begin{remark}[Formal status of the continuous-time limit]
\label[remark]{rem:continuous_time_formal}
The derivation passes from the discrete-time KL decomposition to the continuous-time
path integral~\cref{eq:kl_divergence} via $\Delta t\to 0$.  A fully rigorous treatment
requires two things: (i)~establishing absolute continuity of $Q$ with respect to
$P^\theta$ on path space, so that $dQ/dP^\theta$ is well-defined; and (ii)~verifying
convergence of the discrete KL sum to the path integral.

For~(i), the Radon-Nikodym derivative takes the form of a Dol\'{e}ans-Dade stochastic
exponential of a local martingale, which must be shown to be a true martingale rather
than merely a local one.  For Brownian SDEs, the Girsanov theorem accomplishes this via
the Novikov condition
\begin{equation*}
  \E_P\!\left[\exp\!\left(\tfrac{1}{2}\int_0^T
  \nabla\phi_t(X_t)^\top\mD\,\nabla\phi_t(X_t)\,\diff t\right)\right]<\infty;
\end{equation*}
the exponent matches the integrand of the quadratic KL term in~\cref{eq:kl_divergence}.
For pure-jump processes, the analogous requirement is a Novikov-type condition on the
stochastic exponential now driven by the tilting factor $H_t$ rather than
$\nabla\phi_t$; the complete theory is developed in~\cite{jacod2013limit}.  The
$A_t<0$ constraint ensures the tilted measure has finite total mass and all finite
moments, consistent with such integrability requirements, but a complete formal proof is
left to future work.
\end{remark}

\section{Conditionally Gaussian representations of L\'{e}vy measures}
\label{app:cond_gaussian}

It can be shown that a disintegration of a L\'{e}vy measure $\levy(\diff y)$ can be
formed such that~\citep{rosinski2001series}
\begin{equation*}
    \levy(\diff y) = \int_0^\infty \sigma(r; \diff y)\,\pi(r)\,\diff r,
\end{equation*}
where $\sigma(\cdot;\cdot)$ is a probability kernel for fixed $r$ and $\pi(r)$ is the
L\'{e}vy measure of a subordinator process.  Hence $\levy(\diff y)$ is modelled as a
mixture measure with mixing measure $\pi(r)$.  This implies that a sequence of jumps
$\{y_i\}$ can be represented as conditionally independent random variables given
$\{r_i\}$.

\paragraph{Symmetric stable processes}
For the symmetric $\alpha$-stable case the mixing measure is $\pi(r) = r^{-1-\alpha}$
(corresponding to a non-negative stable subordinator).  Choosing the probability kernel
as a zero-mean Gaussian $\Normal{0}{r^2\sigma_G^2}$~\citep{lemke2015fully}, the
disintegration for a single jump $y$ is
\begin{equation*}
    \levy(y)\,\diff y
    \propto \int_0^\infty \Normal{y}{r^2\sigma_G^2}\,r^{-1-\alpha}\,\diff r\,\diff y.
\end{equation*}
In practice we use a truncated mixing measure $\pi_\tau(r) = \alpha\tau^\alpha\,r^{-1-\alpha}$
(supported on $[\tau,\infty)$), yielding a proper probability density.

\paragraph{Tilted symmetric stable processes}
The L\'{e}vy measure of the tilted process is
$\tlevy(y, t, X_t) = e^{\phi_t(X_t+y) - \phi_t(X_t)}\levy(\diff y)$.  Introducing
the conditionally Gaussian disintegration and the quadratic parametrization
of~\cref{eq:quadratic_phi}, the normalizing constant of the conditional kernel is
\begin{equation*}
    C(r, t, X_t) = \int_{\R}
    e^{\phi_t(X_t+y)-\phi_t(X_t)}\,\Normal{y}{r^2\sigma_G^2}\,\diff y.
\end{equation*}
With $K_1 = 2A_t X_t + B_t$ and $K_2 = A_t - \frac{1}{2r^2\sigma_G^2}$, completing
the square gives the closed form in the main text.  The normalizing
constant for the truncated tilted mixing measure is
$K(\alpha, t, X_t) = \int_\tau^\infty C(r, t, X_t)\,r^{-1-\alpha}\,\diff r$,
and the resulting probability density is
\begin{equation*}
    \tlevy_\tau(y, t, X_t) = \int_\tau^\infty
    \tilde{\sigma}(r; y, t, X_t)\,\tilde{\pi}_\tau(r, t, X_t)\,\diff r,
\end{equation*}
with $\tilde{\sigma}$ and $\tilde{\pi}_\tau$ as defined in~\cref{thm:cond_gaussian}.
One can verify $\int_{\R}\tlevy_\tau(\diff y, t, X_t) = 1$ by exchanging the order of
integration and using the fact that $\tilde{\sigma}$ is a probability kernel.

\section{Exact rejection sampler for the tilted mixing measure}
\label{app:rejection_sampler}

\begin{proof}[Proof of \cref{thm:rejection_sampler}]
We derive a uniform envelope for $C(r,t,X_t)$ under the quadratic
parametrisation~\cref{eq:quadratic_phi}.  From~\cref{rem:tractability},
with $K_1 = 2A_t X_t + B_t$ and $K_2 = A_t - \tfrac{1}{2r^2\sigma_G^2}$,
\begin{equation*}
    C(r,t,X_t) = \frac{1}{\sqrt{-2K_2 r^2\sigma_G^2}}
    \exp\!\left(-\frac{K_1^2}{4K_2}\right).
\end{equation*}
Introduce $u = 2|A_t|r^2\sigma_G^2 \geq 0$.  Since $A_t < 0$, we have
$-2K_2 r^2\sigma_G^2 = 1 + u$ and $-K_1^2/(4K_2) = K_1^2 u\,/\,(4|A_t|(1+u))$, giving
\begin{equation*}
    C(r,t,X_t) = \frac{1}{\sqrt{1+u}}
    \exp\!\left(\frac{K_1^2}{4|A_t|}\cdot\frac{u}{1+u}\right).
\end{equation*}
For all $u \geq 0$, both $1/\sqrt{1+u} \leq 1$ and $u/(1+u) \leq 1$, so
\begin{equation*}
    C(r,t,X_t) \leq \exp\!\left(\frac{K_1^2}{4|A_t|}\right) =: M(t,X_t).
\end{equation*}
The bound $M(t,X_t)$ is independent of $r$ and hence a valid envelope.
Proposing $r^* \sim q(r) \propto r^{-1-\alpha}\,\mathbf{1}_{r \geq \tau}$
via the inverse-CDF transform and accepting with probability
$C(r^*,t,X_t)/M(t,X_t)$ yields an exact draw from
$\tilde{\pi}_\tau(r,t,X_t) \propto C(r,t,X_t)\,r^{-1-\alpha}$
by standard rejection-sampling correctness.
\end{proof}

\vspace{-4mm}
\section{Analytical integration of the KL integrand}
\label{app:kl_integrals}
\vspace{-2mm}

To compute the KL loss component at each time step, the following integral must be
evaluated:
\begin{equation}
\label{eq:target_integral_appendix}
    \mathcal{I}(t, X_t)
    = \int_\tau^\infty \bigl(f(y, t, X_t) + f(-y, t, X_t)\bigr)\,y^{-1-\alpha}\,\diff y,
\end{equation}
where $f(y, t, X_t) = H_t(X_t, y)\ln H_t(X_t, y) - H_t(X_t, y) + 1$ and
$H_t(x, y) = \exp(A_t(2xy + y^2) + B_t y)$ under the quadratic parametrization.
Setting $K_1 = 2A_t X_t + B_t$, the integrand decomposes into four components:
\begin{align*}
    \mathcal{I}_1^+(t, X_t) &= K_1\!\int_\tau^\infty e^{K_1 y + A_t y^2}\,y^{-\alpha}\,\diff y, \\
    \mathcal{I}_2^+(t, X_t) &= A_t\!\int_\tau^\infty e^{K_1 y + A_t y^2}\,y^{1-\alpha}\,\diff y, \\
    \mathcal{I}_3^+(t, X_t) &= \int_\tau^\infty e^{K_1 y + A_t y^2}\,y^{-1-\alpha}\,\diff y, \\
    \mathcal{I}_4^+(t, X_t) &= \int_\tau^\infty y^{-1-\alpha}\,\diff y = \frac{\tau^{-\alpha}}{\alpha}.
\end{align*}
Each of $\mathcal{I}_1^\pm$, $\mathcal{I}_2^\pm$, $\mathcal{I}_3^\pm$ can be expressed
as a convergent series using the substitution $\eta = A_t y^2$ (with $A_t < 0$) and
upper incomplete gamma functions.  Expanding $e^{K_1 y} = \sum_{n=0}^\infty
\frac{K_1^n}{n!}y^n$ yields
\begin{align*}
    \mathcal{I}_1^+(t, X_t) &= \frac{K_1}{2}\sum_{n=0}^\infty
      \frac{K_1^n}{n!}(-A_t)^{-\frac{n-\alpha+1}{2}}
      \,\Gamma\!\left(\tfrac{n-\alpha+1}{2},\,-A_t\tau^2\right),
\end{align*}
and analogously for the remaining terms.  The total integral at each step is
$\mathcal{I}(t, X_t) = \mathcal{I}^+(t, X_t) + \mathcal{I}^-(t, X_t)$, where
$\mathcal{I}^\pm$ denotes the contributions from $+y$ and $-y$ respectively, truncated
to a finite number of series terms in practice.

\paragraph{Comparison with Monte Carlo estimation}
This series representation provides an exact alternative to~\cref{alg:kl_loss_approximation}.  In practice, however, for the small truncation
thresholds $\tau$ used in our experiments, the series estimator exhibits larger variance
than direct Monte Carlo.  When $|K_1|/\sqrt{|A_t|}$ is non-negligible, early terms in
the series can be large in magnitude before the factorial denominator dominates, causing
numerical instability that persists even after series truncation.~\cref{alg:kl_loss_approximation}
is therefore used throughout this paper; the series representation is retained here as a
reference and may be preferable when $\tau$ is large relative to the jump scale.

\vspace{-2mm}
\section{Extension to state-dependent jump sizes}
\label{app:state_dep}
\vspace{-2mm}

The derivation in~\cref{sec:variational_framework} assumes that jump sizes
enter the prior SDE additively and independently of the current state.  Here we show
that the variational framework extends to the more general prior
\begin{equation}
\label{eq:prior_sde_statedep}
    \diff X_t = f_t^\theta(X_t)\,\diff t + \sigma(X_t)\,\diff B_t
              + \int_\R \gamma(X_{t^-})\,y\,\tilde{N}(\diff y,\diff t),
\end{equation}
where $\gamma:\R\to\R_{>0}$ is a smooth state-dependent scaling function and
$\tilde{N}$ is the compensated Poisson random measure with intensity
$\levy_\tau(\diff y)\,\diff t$.  The infinitesimal generator of~\cref{eq:prior_sde_statedep}
is
\begin{equation}
\label{eq:prior_generator_statedep}
    \genP_t G(x)
    = f_t^\theta(x)\,\nabla G(x)
    + \tfrac{1}{2}\mathrm{tr}(\mD\,\nabla^2 G(x))
    + \int_\R\bigl(G(x+\gamma(x)y)-G(x)\bigr)\levy_\tau(\diff y).
\end{equation}

\begin{prop}[Tilted variational family under state-dependent jump sizes]
\label{prop:state_dep}
For a prior with generator~\cref{eq:prior_generator_statedep}, the KL divergence
between the prior and the optimal Markov posterior is
\begin{equation}
\label{eq:kl_statedep}
    \KL{Q}{P^\theta}
    = \E_Q\!\left[
        \int_0^T\!\int_\R f_\gamma(y,t,X_t)\,\levy_\tau(\diff y)\,\diff t
        + \frac{1}{2}\int_0^T \nabla\phi_t(X_t)^\top\mD\,\nabla\phi_t(X_t)\,\diff t
      \right],
\end{equation}
where $f_\gamma(y,t,x) = H_\gamma\ln H_\gamma - H_\gamma + 1$ with
$H_\gamma \equiv H_t^\gamma(x,y) = e^{\phi_t(x+\gamma(x)y)-\phi_t(x)}$.
The posterior generator is
\begin{equation}
\label{eq:tilted_generator_statedep}
\begin{aligned}
    \genQ_t G(x)
    &= \bigl(f_t^\theta(x)+\mD\nabla\phi_t(x)\bigr)^\top\!\nabla G(x)
     + \tfrac{1}{2}\mathrm{tr}(\mD\,\nabla^2 G(x)) \\
    &\quad + \int_\R\bigl(G(x+\gamma(x)y)-G(x)\bigr)\,
             e^{\phi_t(x+\gamma(x)y)-\phi_t(x)}\levy_\tau(\diff y),
\end{aligned}
\end{equation}
and the variational SDE is
\begin{equation}
\label{eq:variational_sde_statedep}
    \diff X_t = \bigl(f_t^\theta(X_t)+\mD\nabla\phi_t(X_t)\bigr)\diff t
              + \sigma(X_t)\,\diff B_t
              + \int_\R \gamma(X_{t^-})\,y\,\tilde{N}^\phi(\diff y,\diff t),
\end{equation}
where $\tilde{N}^\phi$ is the compensated Poisson random measure with tilted intensity
$e^{\phi_t(X_{t^-}+\gamma(X_{t^-})y)-\phi_t(X_{t^-})}\levy_\tau(\diff y)\,\diff t$.
\end{prop}

\begin{proof}
The proof follows Steps~1 and~3 of~\cref{app:variational_derivation} verbatim, with
the sole substitution $G(x+y)\to G(x+\gamma(x)y)$ throughout the jump generator.  The
Brownian contribution is unchanged.  For the jump contribution, the same cancellation
as the jump calculation in Step~3 gives $\int_\R f_\gamma(y,t,x)\,\levy_\tau(\diff y)$ with
$H_\gamma = e^{\phi_t(x+\gamma(x)y)-\phi_t(x)}$ in place of $H$.
\end{proof}

\paragraph{Simulation}
The conditionally Gaussian disintegration of~\cref{thm:cond_gaussian} adapts to the
state-dependent case: the tilted kernel becomes a Gaussian in $y$ with mean and
variance scaled by $\gamma(x)^{-1}$ and $\gamma(x)^{-2}$ respectively.  The quadratic
parametrization and the rejection sampler of~\cref{sec:simulation} remain tractable
under this scaling with straightforward modifications.  All experiments in this
paper use $\gamma(\cdot)\equiv 1$, which recovers the formulation of~\cref{sec:method}; adapting the simulation to $\gamma(\cdot)\neq 1$ is left to
future work.
\vspace{-2mm}
\section{Connection to score-based generative modelling}
\vspace{-2mm}
The variational family has a precise structural connection to score-based generative
modelling. For Gaussian SDEs the drift correction $\mD\nabla\phi_t$ equals the score
of the marginal density, enabling learned reverse processes. The analogous formula for
L\'{e}vy systems replaces the jump measure by $[p_t(x+\vy)/p_t(x)]\,\levy(\diff\vy)$;
at optimum, the tilted measure derived here reduces to exactly this ratio, identifying
neural tilting as the structural backbone of L\'{e}vy diffusion generative models.

Two recent works have constructed L\'{e}vy generative models.
\cite{yoon2023score} derive an exact reverse-time SDE for an isotropic $\alpha$-stable
forward process, in which the learnt correction enters only the reverse drift via a
fractional score function of the current state, whilst the driving noise in the reverse
process remains a fixed isotropic $\alpha$-stable process.
\cite{shariatian2025denoising} take a discrete-time route, replacing Gaussian increments
in DDPM with $\alpha$-stable ones and using a normal variance-mixture representation of
$\alpha$-stable random variables to recover tractable backward kernels; here too the
background jump law is fixed and the generative capability resides in the denoising
function.
A generative model built on the variational family of~\cref{thm:optimal_posterior_family}
would differ structurally: the reverse process would carry a state-conditioned tilted
jump measure $e^{\phi_t(x+\vy)-\phi_t(x)}\,\levy(\diff\vy)$, so the jump law itself
adapts to the current state rather than remaining a fixed isotropic noise corrected by a
learnt drift.

\begin{remark}[Joint parametrization in the jump-diffusion setting]
In the joint setting ($\sigma \neq 0$), the Brownian drift correction $\mD\nabla\phi_t(X_t)$ and the jump tilting factor $e^{\phi_t(X_t+\vy)-\phi_t(X_t)}$ are both determined by the same potential $\phi_t$. This coupling is a structural consequence of~\cref{thm:optimal_posterior_family}: independently parametrizing the two channels breaks the shared $\phi_t$ structure and forfeits the optimality guarantee. Under the quadratic parametrization of~\cref{sec:quadratic}, both the gradient $\nabla\phi_t$ and the increment $\phi_t(x+y)-\phi_t(x)$ are linear functions of the same neural outputs $(A_t, B_t)$, so joint estimation of the Brownian and jump corrections is automatic. Developing richer parametrizations that balance the two channels with greater flexibility while preserving the Markov optimality structure is a natural direction for future work.
\end{remark}

\vspace{-2mm}
\section{Implementation Details}
\label{app:implementation}
\vspace{-2mm}

All tilted stable (TS) and Gaussian SDE experiments were implemented in JAX~\citep{jax2018github}, and parameter updates were carried out using the Optax optimisation library~\citep{deepmind2020jax}.  Financial data were sourced via the \texttt{yfinance} API~\citep{aroussi2019yfinance}.

\subsection{Wall-clock training time}
To give a practical sense of the computational cost, we report representative
wall-clock training times measured on a single NVIDIA GH200 Grace--Hopper GPU.
For the synthetic
experiments, times are reported per dataset realisation; for the financial
experiments, times are reported per rolling train-and-forecast window.  The
TS and Gaussian SDE models use the same optimisation protocol,
so their runtimes are directly comparable and isolate the additional cost of
heavy-tailed posterior inference.  The remaining baselines are included as
context, though they solve different optimisation problems and belong to
different model classes.

\begin{wraptable}[22]{r}{0.6\textwidth}
  \centering
\vspace{-5mm}
\caption{Representative single-GPU wall-clock training times.  Synthetic times
are reported per run (3000 optimisation steps).  Financial times are reported
per rolling window.  The multivariate NJ-SDE time is estimated from the 90 of
302 windows that completed.}
\label{tab:runtime}
\centering
\footnotesize
\setlength{\tabcolsep}{4pt}
\begin{tabular}{l l c}
\toprule
Setting & Model & Time \\
\midrule
Synthetic (OU / double well) & TS (ours) & $\sim 122$ min/run \\
Synthetic (OU / double well) & Gaussian SDE & $\sim 25$--$26$ min/run \\
\midrule
Financial, univariate & TS (ours) & $\sim 84$ min/window \\
Financial, univariate & Gaussian SDE & $\sim 20$ min/window \\
Financial, univariate & DeepAR & $\sim 1.2$ min/window \\
Financial, univariate & NJ-SDE & $\sim 2.4$ min/window \\
Financial, univariate & Neural MJD & $\sim 0.3$ min/window \\
Financial, univariate & N-HiTS & $\sim 0.3$ min/window \\
Financial, univariate & DLinear & $< 1$ s/window \\
\midrule
Financial, multivariate ($d=10$) & TS (ours) & $\sim 247$ min/window \\
Financial, multivariate ($d=10$) & Gaussian SDE & $\sim 34$ min/window \\
Financial, multivariate ($d=10$) & DeepAR & $\sim 3.3$ min/window \\
Financial, multivariate ($d=10$) & NJ-SDE & $\sim 11$ min/window \\
Financial, multivariate ($d=10$) & Neural MJD & $\sim 0.3$ min/window \\
Financial, multivariate ($d=10$) & N-HiTS & $\sim 0.7$ min/window \\
Financial, multivariate ($d=10$) & DLinear & $\sim 0.1$ min/window \\
\bottomrule
\end{tabular}
\end{wraptable}

The additional cost of TS relative to the Gaussian SDE is consistent across
settings and reflects the overhead of sampling from the tilted L\'{e}vy measure
and estimating the jump contribution to the ELBO.  In the synthetic setting,
TS is approximately $5\times$ slower than the Gaussian SDE while remaining
highly stable across both drift families.  In finance, the univariate TS model
requires approximately $84$ minutes per rolling window, increasing to
approximately $247$ minutes in the $d=10$ multivariate setting.  This scaling is
consistent with the added cost of simulating and optimising heavy-tailed jump
structure in higher dimension.

Although TS is substantially slower than forecasting-oriented baselines such as
DLinear, N-HiTS, and Neural MJD, these models do not perform posterior
inference over a L\'{e}vy-driven latent path measure.  We therefore view the
runtime comparison with the Gaussian SDE as the most informative like-for-like
measure of the computational price of modelling heavy-tailed jump dynamics.

\vspace{-1mm}
\subsection{Optimisation details}
\vspace{-1mm}
Both TS and Gaussian SDE models use the same neural drift parametrisation
$f_t^\theta$: a one-hidden-layer MLP of width $32$.  We found that increasing
the drift width consistently degraded predictive performance, suggesting that a
small drift network acts as an effective regulariser by preventing the drift
from absorbing variability that should instead be attributed to the latent noise
process.

\paragraph{Tilted Stable architecture}
The TS variational posterior uses the quadratic parametrisation from
\cref{sec:quadratic}, with the coefficients $A_t$ and $B_t$ generated by MLPs of
width $256$ and $5$ hidden layers.  The temporal encoder uses $100$ learnable
reference times and embedding dimension $64$.  The lower-bound constant in the
reparametrisation of $A_t$ is fixed at $a_{\min}=0.001$.  The truncation threshold
is fixed at $\tau=0.01$ across all experiments.

\begin{remark}[Role and sensitivity of $\tau$]
Setting $\tau>0$ replaces the infinite-activity component of the prior with a finite
compound Poisson process, which is required for the variational framework
of~\cref{sec:variational_framework} to be well-posed without compensators.  The
heavy-tailed character of the prior is not affected: as established
in~\cref{dfn:stable}, the power-law tail and the tail index $\alpha$ are preserved
for any $\tau>0$.  The contribution of the discarded small jumps ($|y|<\tau$) is
well-approximated by Gaussian noise~\citep{AsmussenRosinski2001}, as discussed
in~\cref{rem:truncation}.  Consequently, results are insensitive to the precise
value of $\tau$ provided it is small relative to the typical jump scale; $\tau=0.01$
satisfies this condition in all experimental settings considered here.
\end{remark}

\paragraph{Discretisation and Monte Carlo estimation}
In the synthetic experiments, both TS and Gaussian SDE models are trained for
$3000$ optimisation iterations.  In the financial experiments, the univariate
TS model is also trained for $3000$ iterations per rolling window, while the
multivariate TS model uses $10000$ iterations because convergence is
substantially slower in the $d=10$ setting.  The Gaussian SDE is trained for
$3000$ iterations in all experiments, which we found sufficient for reliable
convergence.  All remaining baselines are likewise trained to convergence, which
in practice was achieved within $3000$ iterations.

In each iteration, the ELBO (or Gaussian analogue) is estimated using $500$
independent simulated latent paths, and the Euler discretisation uses $1000$
latent time steps between the start and end of the observed training window.
For TS, the KL jump integral is approximated using $1000$ Monte Carlo samples
per time step as described in \cref{sec:kl_approx}.  The Gaussian SDE uses the
same Euler discretisation and numbers of simulated paths.

\paragraph{Regularisation}
We apply $\ell_2$ regularisation to all trainable parameters for every model.
The regularisation scale is tuned manually as a model-specific hyperparameter
for all methods considered in the experiments.

\paragraph{Optimiser and gradient stabilisation}
The TS model is trained with RMSProp and no learning-rate decay.  The learning
rate is set to $10^{-4}$ in the synthetic and univariate financial experiments,
and to $10^{-3}$ in the multivariate financial experiment.  We found the absence
of decay important in practice: with decay, the combination of heavy-tailed
Monte Carlo noise and already conservative second-moment normalisation often led
to updates that became too small late in training.  Compared with Adam at
similar learning rates, RMSProp produced materially more stable optimisation,
whereas Adam frequently led to pronounced oscillations.  A plausible
explanation is that under heavy-tailed Monte Carlo noise, first-moment momentum
amplifies rare but very large gradient excursions instead of damping them,
while RMSProp's second-moment normalisation is more conservative.

The Gaussian SDE and all remaining benchmarks are trained with Adam using
learning rate $10^{-4}$ and exponential decay factor $0.95$.

In addition, TS uses a custom layerwise gradient rescaling step before the
RMSProp update.  For each parameter tensor $g$, we compute its robust scale via
the empirical $0.95$-quantile of $|g|$, and compare it with the root-mean-square
norm $\|g\|_2/\sqrt{|g|}$.  The tensor is then rescaled by
\[
    g \leftarrow {g}\,/\,{\max\!\left(1,\,
    \frac{\|g\|_2/\sqrt{|g|}}{q_{0.95}(|g|)+\varepsilon}\right)}.
\]
This transformation preserves the relative magnitudes of coordinates within a
layer, unlike coordinatewise clipping, but suppresses updates when the overall
layerwise norm is dominated by a small number of extreme entries.  Empirically,
this was important for stable optimisation of the TS objective, whose
heavy-tailed jump samples can induce occasional gradient spikes.  We interpret
the method as a robust compromise between no clipping and hard clipping: it
retains directional information while reducing the influence of rare,
disproportionately large gradients.

\begin{table*}[t]
\caption{Jump CRPS ($\sigma_\varepsilon = 0.10$) at increasing percentile thresholds
for each $\alpha$; lower is better.  Best result per metric in \textbf{bold}.}
\label{tab:synthetic_jump}
\centering
\scriptsize
\setlength{\tabcolsep}{2.5pt}
\resizebox{\textwidth}{!}{%
\begin{tabular}{l cc cc cc cc cc cc cc cc}
\toprule
\multicolumn{1}{c}{} & \multicolumn{8}{c}{Ornstein--Uhlenbeck}
& \multicolumn{8}{c}{Double well} \\
\cmidrule(lr){2-9}\cmidrule(lr){10-17}
& \multicolumn{2}{c}{$p_{90}$} & \multicolumn{2}{c}{$p_{95}$}
& \multicolumn{2}{c}{$p_{97.5}$} & \multicolumn{2}{c}{$p_{99}$}
& \multicolumn{2}{c}{$p_{90}$} & \multicolumn{2}{c}{$p_{95}$}
& \multicolumn{2}{c}{$p_{97.5}$} & \multicolumn{2}{c}{$p_{99}$} \\
\cmidrule(lr){2-3}\cmidrule(lr){4-5}\cmidrule(lr){6-7}\cmidrule(lr){8-9}
\cmidrule(lr){10-11}\cmidrule(lr){12-13}\cmidrule(lr){14-15}\cmidrule(lr){16-17}
$\alpha$ & G & TS & G & TS & G & TS & G & TS & G & TS & G & TS & G & TS & G & TS \\
\midrule
$1.1$ & $1.40$ & $\mathbf{1.15}$ & $1.70$ & $\mathbf{1.52}$ & $2.15$ & $\mathbf{2.10}$ & $2.46$ & $\mathbf{2.41}$ & $0.54$ & $\mathbf{0.31}$ & $0.66$ & $\mathbf{0.42}$ & $0.85$ & $\mathbf{0.60}$ & $1.35$ & $\mathbf{1.05}$ \\
$1.2$ & $0.90$ & $\mathbf{0.38}$ & $1.00$ & $\mathbf{0.46}$ & $1.12$ & $\mathbf{0.56}$ & $1.23$ & $\mathbf{0.67}$ & $0.42$ & $\mathbf{0.22}$ & $0.48$ & $\mathbf{0.28}$ & $0.58$ & $\mathbf{0.37}$ & $0.85$ & $\mathbf{0.62}$ \\
$1.3$ & $1.10$ & $\mathbf{0.97}$ & $1.63$ & $\mathbf{1.54}$ & $\mathbf{2.35}$ & $2.36$ & $\mathbf{3.01}$ & $3.20$ & $0.43$ & $\mathbf{0.23}$ & $0.48$ & $\mathbf{0.27}$ & $0.55$ & $\mathbf{0.35}$ & $0.77$ & $\mathbf{0.54}$ \\
$1.4$ & $0.53$ & $\mathbf{0.39}$ & $0.63$ & $\mathbf{0.53}$ & $0.82$ & $\mathbf{0.79}$ & $\mathbf{1.25}$ & $1.26$ & $0.43$ & $\mathbf{0.23}$ & $0.49$ & $\mathbf{0.29}$ & $0.60$ & $\mathbf{0.38}$ & $0.85$ & $\mathbf{0.60}$ \\
$1.5$ & $0.55$ & $\mathbf{0.39}$ & $0.66$ & $\mathbf{0.50}$ & $0.86$ & $\mathbf{0.71}$ & $1.14$ & $\mathbf{0.98}$ & $0.39$ & $\mathbf{0.20}$ & $0.42$ & $\mathbf{0.23}$ & $0.47$ & $\mathbf{0.29}$ & $0.61$ & $\mathbf{0.41}$ \\
$1.6$ & $0.39$ & $\mathbf{0.26}$ & $0.41$ & $\mathbf{0.29}$ & $0.44$ & $\mathbf{0.33}$ & $0.52$ & $\mathbf{0.38}$ & $0.40$ & $\mathbf{0.20}$ & $0.43$ & $\mathbf{0.23}$ & $0.47$ & $\mathbf{0.27}$ & $0.59$ & $\mathbf{0.38}$ \\
$1.7$ & $0.43$ & $\mathbf{0.31}$ & $0.47$ & $\mathbf{0.38}$ & $0.52$ & $\mathbf{0.43}$ & $0.58$ & $\mathbf{0.48}$ & $0.40$ & $\mathbf{0.21}$ & $0.43$ & $\mathbf{0.23}$ & $0.47$ & $\mathbf{0.28}$ & $0.56$ & $\mathbf{0.38}$ \\
$1.8$ & $0.84$ & $\mathbf{0.24}$ & $0.85$ & $\mathbf{0.26}$ & $0.91$ & $\mathbf{0.28}$ & $0.98$ & $\mathbf{0.34}$ & $0.41$ & $\mathbf{0.20}$ & $0.43$ & $\mathbf{0.23}$ & $0.47$ & $\mathbf{0.26}$ & $0.55$ & $\mathbf{0.34}$ \\
$1.9$ & $0.45$ & $\mathbf{0.24}$ & $0.48$ & $\mathbf{0.27}$ & $0.51$ & $\mathbf{0.31}$ & $0.54$ & $\mathbf{0.33}$ & $0.41$ & $\mathbf{0.21}$ & $0.43$ & $\mathbf{0.23}$ & $0.46$ & $\mathbf{0.26}$ & $0.52$ & $\mathbf{0.32}$ \\
\midrule
All $\alpha$ & $0.73$ & $\mathbf{0.48}$ & $0.87$ & $\mathbf{0.64}$ & $1.08$ & $\mathbf{0.87}$ & $1.30$ & $\mathbf{1.12}$ & $0.43$ & $\mathbf{0.22}$ & $0.47$ & $\mathbf{0.27}$ & $0.55$ & $\mathbf{0.34}$ & $0.74$ & $\mathbf{0.52}$ \\
\bottomrule
\end{tabular}
}
\vspace{-4mm}
\end{table*}
\begin{table*}[t]
\caption{Synthetic experiment results ($\sigma_\varepsilon = 0.05$), reporting means
over 50 realisations per $\alpha$.  Held-out CRPS and mean absolute parameter recovery errors
for the OU and double-well (DW) systems; lower is better.  Best result per metric in
\textbf{bold}.}
\label{tab:synthetic_app_main}
\centering
\scriptsize
\setlength{\tabcolsep}{3pt}
\begin{tabular}{l cc cc cc cc cc cc}
\toprule
& \multicolumn{6}{c}{Ornstein--Uhlenbeck}
& \multicolumn{6}{c}{Double well} \\
\cmidrule(lr){2-7}\cmidrule(lr){8-13}
& \multicolumn{2}{c}{CRPS $\downarrow$}
& \multicolumn{2}{c}{$|\hat{\theta}{-}\theta^*|$ $\downarrow$}
& \multicolumn{2}{c}{$|\hat{\mu}{-}\mu^*|$ $\downarrow$}
& \multicolumn{2}{c}{CRPS $\downarrow$}
& \multicolumn{2}{c}{$|\hat{\theta}_1{-}\theta_1^*|$ $\downarrow$}
& \multicolumn{2}{c}{$|\hat{\theta}_2{-}\theta_2^*|$ $\downarrow$} \\
\cmidrule(lr){2-3}\cmidrule(lr){4-5}\cmidrule(lr){6-7}
\cmidrule(lr){8-9}\cmidrule(lr){10-11}\cmidrule(lr){12-13}
$\alpha$ & G & TS & G & TS & G & TS & G & TS & G & TS & G & TS \\
\midrule
$1.1$ & $0.99$ & $\mathbf{0.63}$ & $5.13$ & $\mathbf{2.17}$ & $\mathbf{3.99}$ & $7.07$ & $0.41$ & $\mathbf{0.15}$ & $1.10$ & $\mathbf{0.83}$ & $1.93$ & $\mathbf{0.24}$ \\
$1.2$ & $0.81$ & $\mathbf{0.30}$ & $7.59$ & $\mathbf{2.55}$ & $\mathbf{3.16}$ & $3.80$ & $0.36$ & $\mathbf{0.13}$ & $\mathbf{0.92}$ & $1.10$ & $2.17$ & $\mathbf{0.29}$ \\
$1.3$ & $0.55$ & $\mathbf{0.38}$ & $9.01$ & $\mathbf{2.42}$ & $\mathbf{2.67}$ & $5.14$ & $0.34$ & $\mathbf{0.12}$ & $1.73$ & $\mathbf{1.39}$ & $3.35$ & $\mathbf{0.33}$ \\
$1.4$ & $0.42$ & $\mathbf{0.25}$ & $7.48$ & $\mathbf{1.91}$ & $\mathbf{2.60}$ & $5.12$ & $0.36$ & $\mathbf{0.12}$ & $\mathbf{1.42}$ & $1.62$ & $3.26$ & $\mathbf{0.31}$ \\
$1.5$ & $0.39$ & $\mathbf{0.23}$ & $7.18$ & $\mathbf{2.64}$ & $\mathbf{2.83}$ & $4.26$ & $0.33$ & $\mathbf{0.11}$ & $\mathbf{1.16}$ & $1.88$ & $3.42$ & $\mathbf{0.44}$ \\
$1.6$ & $0.35$ & $\mathbf{0.21}$ & $6.16$ & $\mathbf{2.41}$ & $\mathbf{3.54}$ & $3.74$ & $0.34$ & $\mathbf{0.12}$ & $\mathbf{0.98}$ & $1.82$ & $3.08$ & $\mathbf{0.30}$ \\
$1.7$ & $0.36$ & $\mathbf{0.22}$ & $9.61$ & $\mathbf{2.65}$ & $\mathbf{2.20}$ & $2.57$ & $0.36$ & $\mathbf{0.13}$ & $\mathbf{0.82}$ & $2.25$ & $2.91$ & $\mathbf{0.44}$ \\
$1.8$ & $0.80$ & $\mathbf{0.22}$ & $8.57$ & $\mathbf{2.51}$ & $\mathbf{2.05}$ & $3.22$ & $0.40$ & $\mathbf{0.15}$ & $\mathbf{1.02}$ & $1.26$ & $2.21$ & $\mathbf{0.29}$ \\
$1.9$ & $0.41$ & $\mathbf{0.20}$ & $6.98$ & $\mathbf{1.68}$ & $\mathbf{2.02}$ & $2.56$ & $0.39$ & $\mathbf{0.14}$ & $1.72$ & $\mathbf{1.32}$ & $2.37$ & $\mathbf{0.51}$ \\
\midrule
All $\alpha$ & $0.56$ & $\mathbf{0.29}$ & $7.52$ & $\mathbf{2.33}$ & $\mathbf{2.78}$ & $4.16$ & $0.37$ & $\mathbf{0.13}$ & $\mathbf{1.21}$ & $1.50$ & $2.75$ & $\mathbf{0.35}$ \\
\bottomrule
\end{tabular}%
\vspace{-3mm}
\end{table*}

\vspace{-2mm}
\section{Additional results on synthetic data}
\label{app:extra_synthetic}
\vspace{-2mm}

\begin{table*}[b]
\caption{Jump CRPS ($\sigma_\varepsilon = 0.05$) at increasing percentile thresholds
for each $\alpha$; lower is better.  Best result per metric in \textbf{bold}.}
\label{tab:synthetic_app_jump}
\centering
\scriptsize
\setlength{\tabcolsep}{2.5pt}
\resizebox{\textwidth}{!}{%
\begin{tabular}{l cc cc cc cc cc cc cc cc}
\toprule
\multicolumn{1}{c}{} & \multicolumn{8}{c}{Ornstein--Uhlenbeck}
& \multicolumn{8}{c}{Double well} \\
\cmidrule(lr){2-9}\cmidrule(lr){10-17}
& \multicolumn{2}{c}{$p_{90}$} & \multicolumn{2}{c}{$p_{95}$}
& \multicolumn{2}{c}{$p_{97.5}$} & \multicolumn{2}{c}{$p_{99}$}
& \multicolumn{2}{c}{$p_{90}$} & \multicolumn{2}{c}{$p_{95}$}
& \multicolumn{2}{c}{$p_{97.5}$} & \multicolumn{2}{c}{$p_{99}$} \\
\cmidrule(lr){2-3}\cmidrule(lr){4-5}\cmidrule(lr){6-7}\cmidrule(lr){8-9}
\cmidrule(lr){10-11}\cmidrule(lr){12-13}\cmidrule(lr){14-15}\cmidrule(lr){16-17}
$\alpha$ & G & TS & G & TS & G & TS & G & TS & G & TS & G & TS & G & TS & G & TS \\
\midrule
$1.1$ & $1.40$ & $\mathbf{1.15}$ & $1.70$ & $\mathbf{1.52}$ & $2.15$ & $\mathbf{2.10}$ & $2.46$ & $\mathbf{2.41}$ & $0.55$ & $\mathbf{0.27}$ & $0.67$ & $\mathbf{0.38}$ & $0.89$ & $\mathbf{0.57}$ & $1.41$ & $\mathbf{1.03}$ \\
$1.2$ & $0.90$ & $\mathbf{0.38}$ & $1.00$ & $\mathbf{0.46}$ & $1.12$ & $\mathbf{0.56}$ & $1.23$ & $\mathbf{0.67}$ & $0.46$ & $\mathbf{0.22}$ & $0.55$ & $\mathbf{0.30}$ & $0.72$ & $\mathbf{0.44}$ & $1.12$ & $\mathbf{0.80}$ \\
$1.3$ & $1.10$ & $\mathbf{0.97}$ & $1.63$ & $\mathbf{1.54}$ & $\mathbf{2.35}$ & $2.36$ & $\mathbf{3.01}$ & $3.20$ & $0.41$ & $\mathbf{0.19}$ & $0.47$ & $\mathbf{0.24}$ & $0.57$ & $\mathbf{0.34}$ & $0.82$ & $\mathbf{0.57}$ \\
$1.4$ & $0.53$ & $\mathbf{0.39}$ & $0.63$ & $\mathbf{0.53}$ & $0.82$ & $\mathbf{0.79}$ & $\mathbf{1.25}$ & $1.26$ & $0.43$ & $\mathbf{0.18}$ & $0.49$ & $\mathbf{0.24}$ & $0.60$ & $\mathbf{0.33}$ & $0.85$ & $\mathbf{0.55}$ \\
$1.5$ & $0.55$ & $\mathbf{0.39}$ & $0.66$ & $\mathbf{0.50}$ & $0.86$ & $\mathbf{0.71}$ & $1.14$ & $\mathbf{0.98}$ & $0.39$ & $\mathbf{0.17}$ & $0.44$ & $\mathbf{0.21}$ & $0.52$ & $\mathbf{0.29}$ & $0.72$ & $\mathbf{0.47}$ \\
$1.6$ & $0.39$ & $\mathbf{0.26}$ & $0.41$ & $\mathbf{0.29}$ & $0.44$ & $\mathbf{0.33}$ & $0.52$ & $\mathbf{0.38}$ & $0.37$ & $\mathbf{0.16}$ & $0.40$ & $\mathbf{0.19}$ & $0.46$ & $\mathbf{0.24}$ & $0.58$ & $\mathbf{0.36}$ \\
$1.7$ & $0.43$ & $\mathbf{0.31}$ & $0.47$ & $\mathbf{0.38}$ & $0.52$ & $\mathbf{0.43}$ & $0.58$ & $\mathbf{0.48}$ & $0.41$ & $\mathbf{0.18}$ & $0.44$ & $\mathbf{0.21}$ & $0.50$ & $\mathbf{0.27}$ & $0.64$ & $\mathbf{0.40}$ \\
$1.8$ & $0.84$ & $\mathbf{0.24}$ & $0.85$ & $\mathbf{0.26}$ & $0.91$ & $\mathbf{0.28}$ & $0.98$ & $\mathbf{0.34}$ & $0.43$ & $\mathbf{0.18}$ & $0.46$ & $\mathbf{0.21}$ & $0.51$ & $\mathbf{0.26}$ & $0.62$ & $\mathbf{0.38}$ \\
$1.9$ & $0.45$ & $\mathbf{0.24}$ & $0.48$ & $\mathbf{0.27}$ & $0.51$ & $\mathbf{0.31}$ & $0.54$ & $\mathbf{0.33}$ & $0.42$ & $\mathbf{0.18}$ & $0.45$ & $\mathbf{0.21}$ & $0.49$ & $\mathbf{0.26}$ & $0.60$ & $\mathbf{0.38}$ \\
\midrule
All $\alpha$ & $0.73$ & $\mathbf{0.48}$ & $0.87$ & $\mathbf{0.64}$ & $1.08$ & $\mathbf{0.87}$ & $1.30$ & $\mathbf{1.12}$ & $0.43$ & $\mathbf{0.19}$ & $0.49$ & $\mathbf{0.24}$ & $0.58$ & $\mathbf{0.33}$ & $0.82$ & $\mathbf{0.55}$ \\
\bottomrule
\end{tabular}%
}
\end{table*}
\cref{tab:synthetic_jump,tab:synthetic_app_main,tab:synthetic_app_jump} report the
full synthetic jump-CRPS results at $\sigma_\varepsilon = 0.10$ and the
full synthetic results at $\sigma_\varepsilon = 0.05$.  The qualitative findings
mirror those at $\sigma_\varepsilon = 0.10$: our model achieves lower CRPS than the
Gaussian baseline at every $\alpha$ in both systems, substantially reduces $\theta$
and $\theta_2$ recovery error, and leads on jump CRPS at all thresholds.  The
$\mu$ (OU) and $\theta_1$ (double well) recovery results follow the same pattern as
at $\sigma_\varepsilon = 0.10$: the Gaussian baseline performs better on these
parameters at most $\alpha$ values, reflecting the trade-off between the tilting
function and the global drift parameters discussed in~\cref{sec:exp_synthetic}.  With lower observation noise the jump CRPS
advantage for the double-well system is now fully resolved (no missing entries).
\Cref{fig:posterior_trajectories_all} shows posterior sample paths for nine additional realisations, further illustrating the consistent advantage of the tilted-stable model across seeds and both dynamical systems.

\vspace{-2mm}
\section{Additional financial results}
\label{app:extra_financial}
\vspace{-2mm}
\cref{tab:googl_main} reports results for GOOGL, using the same evaluation
protocol as NVDA (304 rolling windows, 2-day forecast horizon).  N-HiTS was not
evaluated on GOOGL.  The ranking is consistent with NVDA: our model achieves the
lowest CRPS and leads on jump CRPS at every threshold.  Reliability diagrams for GOOGL and the multivariate ($d{=}10$) setting are shown
in~\cref{fig:reliability_financial}.

\begin{table*}[h]
\caption{GOOGL financial forecasting results averaged over 304 evaluation windows;
lower is better for all metrics.  Best in \textbf{bold}.
$^\dagger$Deterministic model; CRPS equals MAE.
Jump CRPS is computed on price increments exceeding the indicated percentile threshold.
N-HiTS was not evaluated on GOOGL.}
\label{tab:googl_main}
\centering
\begin{tabular}{l c cccc}
& & \multicolumn{4}{c}{Jump CRPS $\downarrow$} \\
\cmidrule(lr){3-6}
Model & CRPS $\downarrow$ & $p_{90}$ & $p_{95}$ & $p_{97.5}$ & $p_{99}$ \\
\midrule
NJ-SDE             & $1.199{\pm}1.258$           & $1.609$          & $1.805$          & $2.312$          & $2.570$          \\
DLinear$^\dagger$  & $0.745{\pm}0.609$           & $1.257$          & $1.498$          & $1.919$          & $2.240$          \\
Gaussian SDE       & $0.610{\pm}0.670$           & $1.166$          & $1.369$          & $1.690$          & $1.827$          \\
Neural MJD         & $0.542{\pm}0.454$           & $0.999$          & $1.184$          & $1.455$          & $1.881$          \\
DeepAR             & $0.536{\pm}0.512$           & $1.002$          & $1.223$          & $1.529$          & $1.823$          \\
\textbf{TS (ours)} & $\mathbf{0.435}{\pm}0.419$ & $\mathbf{0.843}$ & $\mathbf{1.053}$ & $\mathbf{1.430}$ & $\mathbf{1.709}$ \\
\bottomrule
\end{tabular}
\vspace{-4mm}
\end{table*}

\begin{figure*}[p]
\centering
\includegraphics[width=\textwidth]{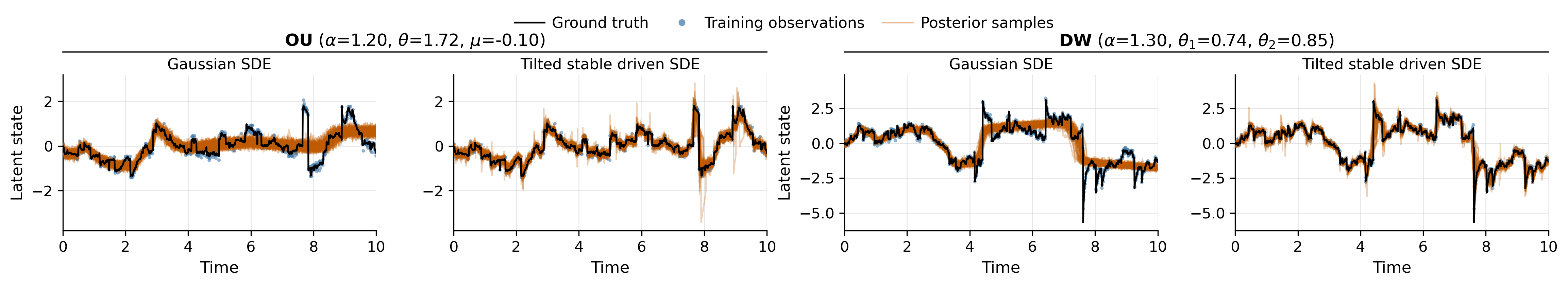}\\[-3mm]
\includegraphics[width=\textwidth]{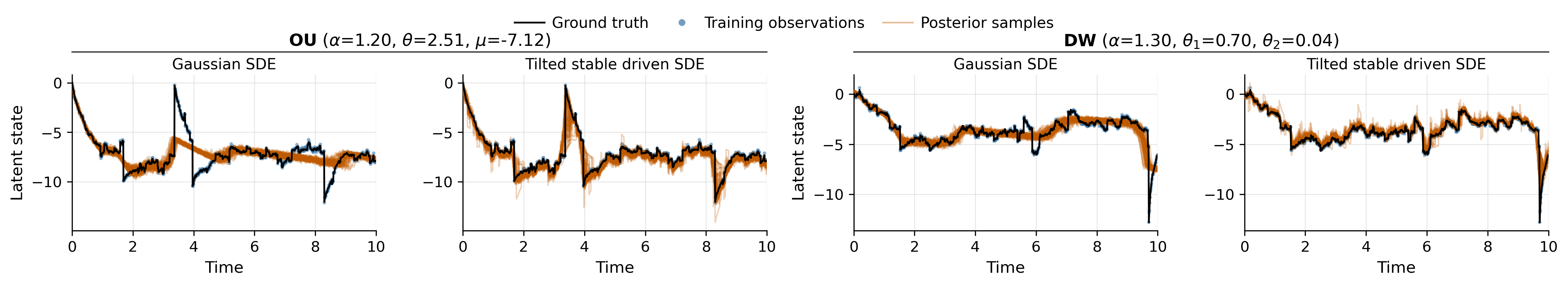}\\[-3mm]
\includegraphics[width=\textwidth]{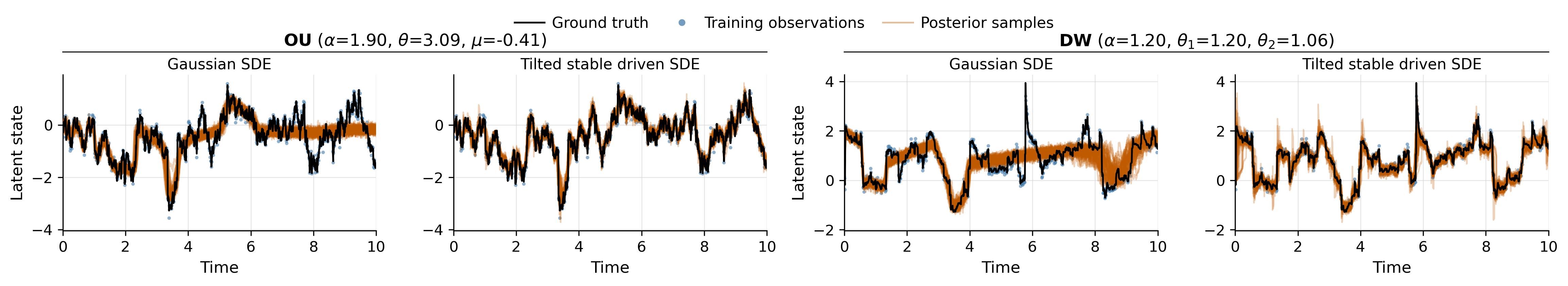}\\[-3mm]
\includegraphics[width=\textwidth]{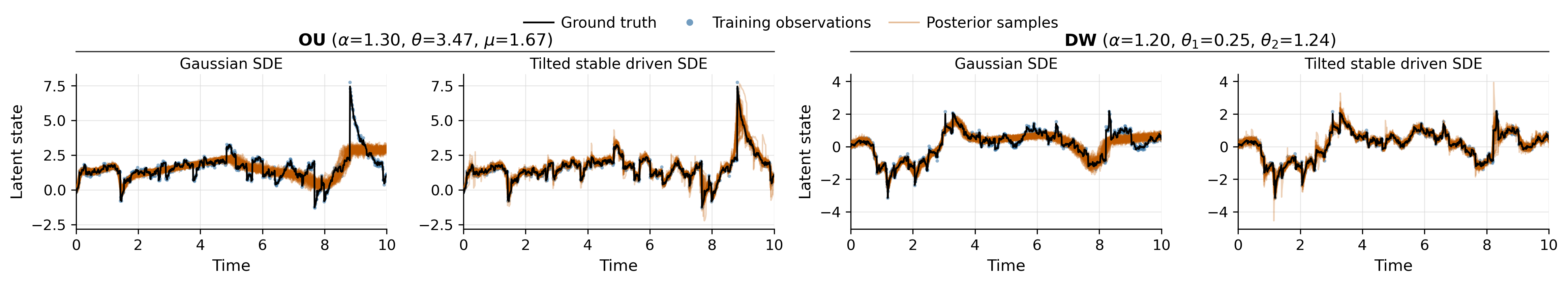}\\[-3mm]
\includegraphics[width=\textwidth]{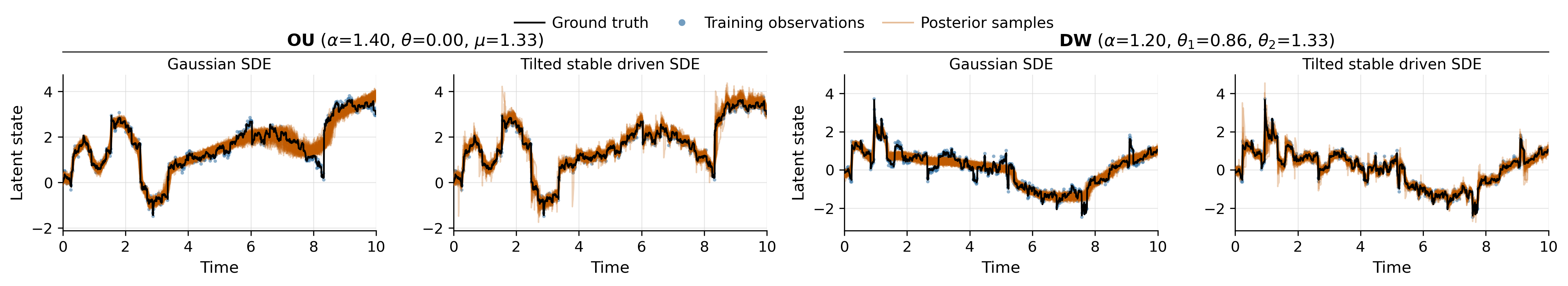}\\[-3mm]
\includegraphics[width=\textwidth]{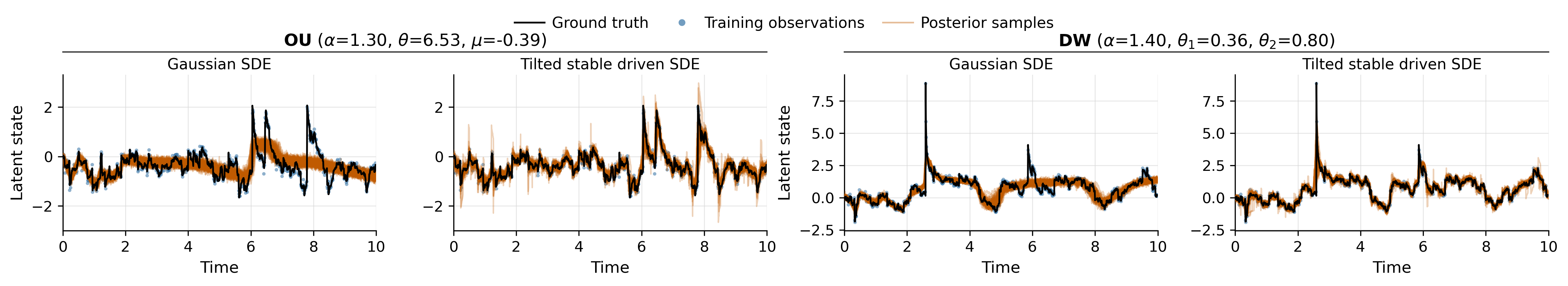}\\[-3mm]
\includegraphics[width=\textwidth]{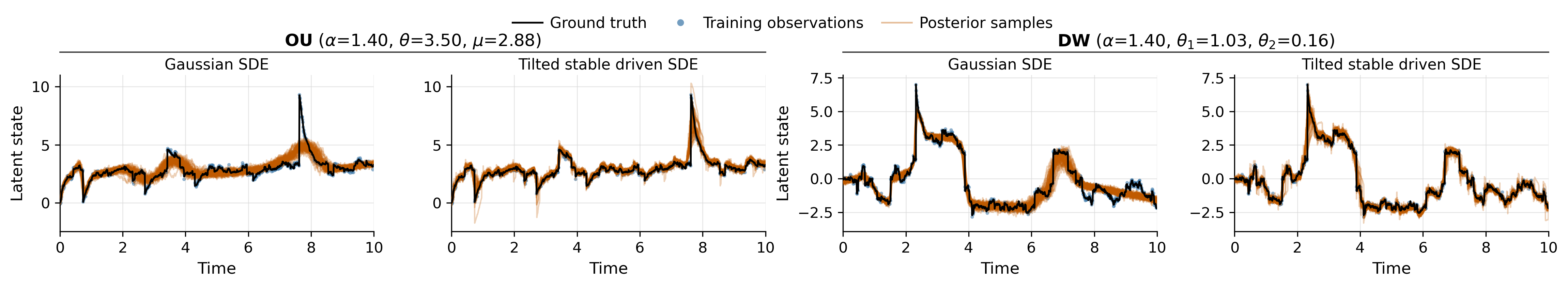}\\[-3mm]
\includegraphics[width=\textwidth]{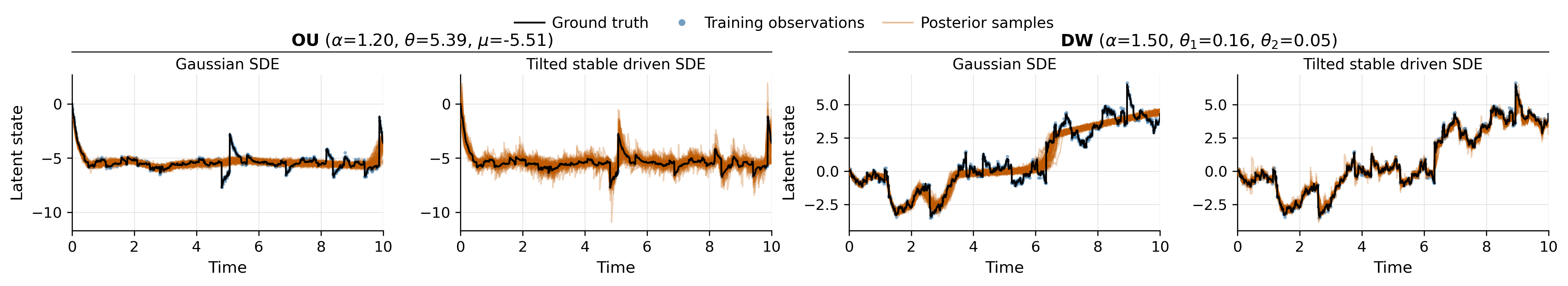}\\[-3mm]
\includegraphics[width=\textwidth]{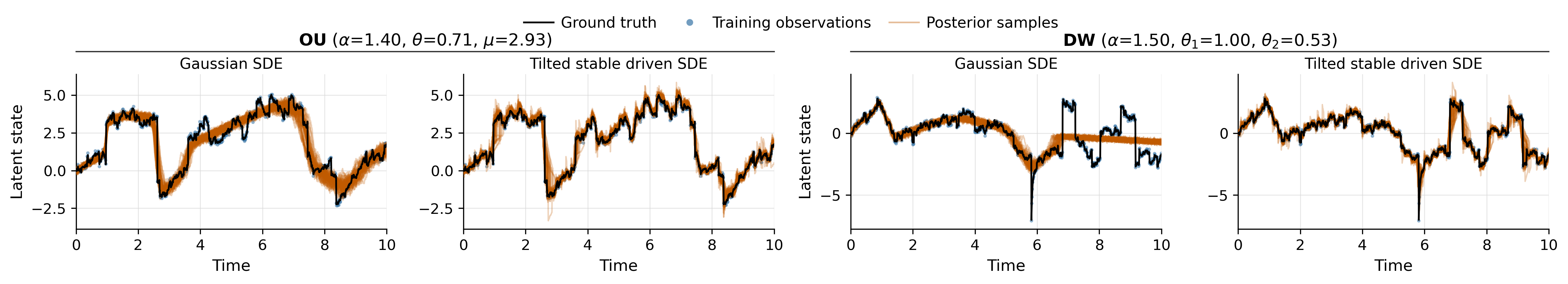}
\caption{Posterior sample paths for nine additional synthetic realisations, comparing the Gaussian SDE and the tilted-stable model on the OU (left pair) and double-well (right pair) systems. Each row corresponds to an independent realisation; the layout within each panel follows~\cref{fig:synthetic_posterior}.}
\label{fig:posterior_trajectories_all}
\end{figure*}

\end{document}